\documentclass[conference]{IEEEtran}
\IEEEoverridecommandlockouts

\usepackage{amsmath,amssymb,amsthm}
\usepackage{graphicx}
\usepackage{cite}
\usepackage{bm}
\usepackage{tikz}
\usepackage{booktabs}
\usepackage{url}
\usepackage{placeins}
\usepackage{stfloats}
\usetikzlibrary{arrows.meta,calc}

\graphicspath{{../figures/}{./}}

\newtheorem{theorem}{Theorem}
\newtheorem{proposition}{Proposition}
\newtheorem{lemma}{Lemma}
\newtheorem{corollary}{Corollary}
\newtheorem{remark}{Remark}
\newtheorem{assumption}{Assumption}

\newcommand{\R}{\mathbb{R}}
\newcommand{\Rot}{R}
\newcommand{\norm}[1]{\left\lVert #1 \right\rVert}

\begin{document}

\title{Anchor-Free Hidden-Target Seeking via Certified Self-Calibration
under Correlated Odometry\thanks{\raggedright Code and data: \protect\url{https://github.com/yashbagla321/gps-free-hidden-target-seeking}, archived at \protect\url{https://doi.org/10.5281/zenodo.22046946}.}}

\author{\IEEEauthorblockN{Yash Bagla}
\IEEEauthorblockA{Michigan State University\\
{\tt\small yashbagla321@gmail.com}}}

\maketitle
\bstctlcite{IEEEexample:BSTctlDash}

\begin{abstract}
We study hidden-target seeking in an anchor-free regime where neither the
vehicle, the relay, nor the target has an accessible global pose. The
vehicle never senses the target directly; it receives only range-bearing
observations through a single relay of unknown position and orientation,
while its own motion is available only through integrating body-frame
odometry. We show that absolute localization is fundamentally impossible:
the joint configuration retains an exact three-dimensional $SE(2)$ gauge
that no estimator can resolve. Yet the quantities required for control
remain fully recoverable: motion collapses the calibration
problem to a task-relevant quotient, the relay-to-odometry yaw together
with the target displacement, for which two distinct vehicle views are
necessary and sufficient to identify in closed form. Building on this
structure, we introduce an $O(K)$ multi-view self-calibration estimator
together with an exact first-order yaw-uncertainty certificate for the
stated translational random-walk odometry model, explicitly propagating
the cross-view correlations integrated odometry induces: the full
certificate attains $95.0\%$ pooled coverage at the nominal $95\%$ level,
while treating correlated poses as independent yields only $82.1\%$. The
certificate drives an end-to-end hybrid policy that excites until
calibration is trustworthy, refuses uncertified estimates, seeks using
continuously re-measured relative geometry, and detects relay-frame
changes through persistent certified inconsistency, so task error does not
inherit unbounded dead-reckoning drift. Across $200$ randomized
closed-loop trials, the proposed method attains $0.064$\,m median station
error versus $0.065$\,m for an oracle given the true relay yaw, despite
$27$\,m median dead-reckoning drift in long-horizon experiments; $90$
physics-based ROS~2/Gazebo trials retain $30/30$ success under nominal
operation, combined communication degradation, and relay-frame
disturbances. These results establish that global localization is
unnecessary for reliable hidden-target seeking under the stated
single-relay model, even when both the sensing infrastructure and the
vehicle's reference frame are uncalibrated.
\end{abstract}


\section{Introduction}
\label{sec:intro}

A robot operating in smoke, clutter, an urban canyon, or under jamming may
be denied two things at once: its own global position and any direct view
of its goal. This paper studies the extreme point of that regime. A single
cooperative relay observes both the robot and a hidden target and reports
range and bearing packets in its own local frame; the relay's position and
yaw are unknown, the robot's global pose is unknown, and the only
proprioceptive signal is integrating body-frame odometry. The question is
whether the robot can still reach the target, with what guarantees, and
with what honest statement of uncertainty. The key result is that losing
global localization does not destroy the task: it removes only
information the controller never needs.

Classical answers each remove a different assumption but not this
combination. Source seeking reaches a target without position information,
yet requires the robot to sense a signal emitted by the source
itself~\cite{zhang2007extremum,cochran2009threed}. Single-beacon
navigation exploits trajectory excitation to localize the robot against
one acoustic beacon, but the beacon anchors the robot rather than relaying
a hidden third party~\cite{batista2011single,webster2012advances}.
Cooperative localization and SLAM estimate relative states while carrying
known unobservable directions~\cite{roumeliotis2002distributed,
cadena2016slam}. One layer above all of these sit motion planners that
consume a localized goal and a calibrated uncertainty model; for example,
chance-constrained receding-horizon planners allocate collision risk under
exactly the kind of state uncertainty our certificate
quantifies~\cite{bagla2023receding}. Hidden-target seeking must first
manufacture that goal estimate, with a trustworthy confidence statement,
from partially calibrated infrastructure.

Two preceding studies established the sensing model. The first proved
trajectory-induced identifiability of a hidden target through an
unknown-pose range and bearing relay when the vehicle trajectory is
globally known~\cite{bagla2026identifiability}. The second closed the loop
in that globally anchored setting with excitation
supervision~\cite{bagla2026closedloop}. The present paper removes the
anchor entirely: the vehicle trajectory is itself latent, known only
through noisy integrating odometry, so the global frame becomes a genuine
gauge and the uncertainty of the calibration is dominated by correlated
pose errors that neither prior formulation had to model.

Rather than reconstructing an inherently unobservable global state, we
identify, certify, and control directly on the smallest observable
quotient that is sufficient for the task. The contributions are fourfold.
\begin{enumerate}
\item \textbf{Anchor-free certified quotient control without accumulated
localization error.} We establish an end-to-end seeking architecture
needing no accessible global pose for the vehicle, relay, or target: it
excites only until the task quotient is certified, refuses uncertifiable
calibration, regulates the target directly in quotient coordinates, and
detects relay-frame changes using the same uncertainty model. Its
ultimate task-error bound is independent of mission duration, and in
closed loop it is statistically indistinguishable from an oracle given
the true relay yaw despite tens of meters of dead-reckoning drift
(Theorem~\ref{thm:seeking} and
Propositions~\ref{prop:acquisition}-\ref{prop:recovery}).
\item \textbf{Global impossibility, task-level recoverability.} We
characterize the exact $SE(2)$ gauge of the vehicle, relay, and target
configuration (Theorem~\ref{thm:gauge}) and show that although absolute
localization is impossible, every quantity required for target-directed
control survives on a low-dimensional task quotient, given the odometric
trajectory; metric ego-motion is necessary for any descent guarantee
(Corollary~\ref{cor:necessity}).
\item \textbf{Anchor-free quotient recovery and registration.} The
anchored two-view calibration identity of~\cite{bagla2026identifiability}
remains exact after quotienting the inaccessible global frame: one
nonzero displacement is necessary and sufficient to recover the task
quotient in closed form (Theorem~\ref{thm:twoview}). We then derive an
$O(K)$ weighted multi-view registration estimator for noisy odometric
poses (Proposition~\ref{prop:registration}).
\item \textbf{Correlated-odometry uncertainty certification.} We derive
the exact first-order calibration uncertainty induced by integrated
translational odometry, retain the full cross-view correlation structure,
and exploit its random-walk form to reduce the resulting dense covariance
evaluation to an $O(K)$ suffix sum (Theorem~\ref{thm:certificate}).
Realized coverage matches nominal levels across the certification
operating regime (Remark~\ref{rem:coverage}); an ablation shows that
treating cross-view poses as independent degrades coverage to $82.1\%$
and ignoring pose uncertainty entirely to $69.3\%$, so the correlated
treatment is not mathematical decoration.
\end{enumerate}

\section{Related Work}
\label{sec:related}

\textbf{Seeking without position.} Extremum seeking steers a vehicle to
the maximizer of a measured field without localization, in the plane and
in three dimensions~\cite{zhang2007extremum,cochran2009threed}. The
vehicle must sense the source directly; our robot senses nothing from the
target, and all target information transits a relay of unknown pose.

\textbf{Single-beacon and range-only navigation.} Observability of
single-range navigation and its filter designs are classical results in
marine robotics~\cite{batista2011single,webster2012advances}, and
range-only relative pose between robots is recoverable with
odometry~\cite{zhou2008robot}. In each case the beacon serves the robot's
own state estimate. Here the relay observes a hidden third party, and the
robot's own absolute state is provably not recoverable at all.

\textbf{Target and sensor self-calibration.} Range--angle sequences have
been used to recover a moving target together with the positions and
synchronization offsets of at least two sensors~\cite{jia2025target}, while
recent multistatic radar work jointly estimates targets and multiple
uncertain receiver poses from bistatic range--bearing data and global pose
priors~\cite{musallam2026target}. Online SLAM can also select informative
trajectory segments to recalibrate camera intrinsics and detect parameter
changes~\cite{keivan2015online}. Those formulations resolve different
ambiguities: here there is one relay, no global relay-pose prior, and only
correlated ego odometry. Absolute states therefore retain an $SE(2)$ gauge;
we estimate and certify the control-relevant quotient and adopt relay-frame
changes only from certified windows.

\textbf{Cooperative localization and gauge-aware estimation.} Distributed
cooperative localization~\cite{roumeliotis2002distributed,
fox2000collaborative}, observability analyses of mobile-robot
estimators~\cite{martinelli2005observability}, and consistency-aware SLAM
that respects unobservable directions~\cite{huang2010observability,
barrau2017invariant,cadena2016slam} all inform our treatment. We push the
gauge view to its limit: estimation is performed directly on the quotient,
and the certificate quantifies exactly the variables the controller
consumes. Adaptive source localization with unknown beacons requires more
beacons than the workspace dimension under range-only
sensing~\cite{guler2017adaptive}; a single unknown-pose relay suffices
here because motion supplies the missing calibration. General Procrustes
uncertainty under arbitrary covariance~\cite{lourenco2017uncertainty} is
treated abstractly; our contribution is the specific $O(K)$ reduction
under odometry's random-walk correlation structure, evaluated online and
gated by closed-loop certification rather than the general
dense-covariance case.

\textbf{Prior work on this sensing model.} The relay packet model was
introduced with a globally known vehicle
trajectory~\cite{bagla2026identifiability}, where two distinct viewpoints
resolve the relay yaw and the target. Excitation-supervised closed-loop
seeking in the same anchored setting followed
in~\cite{bagla2026closedloop}. Neither treats the trajectory as latent.
Removing the anchor changes the mathematics in three ways: the gauge grows
to the full $SE(2)$ orbit, the calibration variable becomes the
relay-to-odometry yaw rather than a globally referenced angle, and, most
consequentially, the estimator's uncertainty acquires correlated pose
errors whose exact first-order treatment
(Theorem~\ref{thm:certificate}) is the technical core of this paper and
has no counterpart in either preprint.

\begin{figure}[t]
\centering
\begin{tikzpicture}[>=Latex,scale=0.8,every node/.style={font=\footnotesize}]
  \coordinate (xo) at (0.3,0.3);
  \coordinate (sa) at (1.6,2.5);
  \coordinate (sb) at (4.4,2.1);
  \coordinate (po) at (5.0,0.2);
  \draw[->,gray] (-0.4,-0.5)--(0.9,-0.5) node[right] {$o_x$};
  \draw[->,gray] (-0.4,-0.5)--(-0.4,0.8) node[above] {$o_y$};
  \node[below left] at (-0.4,-0.5) {$\mathcal O$};
  \filldraw[black] (xo) circle (1.6pt) node[below left] {relay $x_o$};
  \draw[->,thick] (xo)--++({1.0*cos(28)},{1.0*sin(28)}) node[right] {$b_x$};
  \draw[->,thick] (xo)--++({0.8*cos(118)},{0.8*sin(118)}) node[above] {$b_y$};
  \draw[->] ($(xo)+(0.55,0)$) arc[start angle=0,end angle=28,radius=0.55];
  \node at ($(xo)+(0.78,0.16)$) {$\theta$};
  \filldraw[white,draw=black,thick] (sa) circle (2.2pt) node[above] {$\hat s_a$};
  \filldraw[white,draw=black,thick] (sb) circle (2.2pt) node[above] {$\hat s_b$};
  \filldraw[black] (po) rectangle +(0.1,0.1) node[below right] {target $p_o$};
  \draw[->,very thick] (sa)--(sb) node[midway,above] {odometry};
  \draw[dashed,->] (xo)--(sa) node[pos=0.55,left] {$\Rot(\theta)\ell_a^v$};
  \draw[dashed,->] (xo)--(sb) node[pos=0.6,below] {$\Rot(\theta)\ell_b^v$};
  \draw[dashed,->] (xo)--(po) node[pos=0.62,below] {$\Rot(\theta)\ell^t$};
  \draw[->,blue!60!black,thick] (sb)--(po) node[midway,right] {$e_b$};
\end{tikzpicture}
\caption{The odometry frame $\mathcal O$, the unknown-yaw relay frame
$\mathcal B$, and the hidden target. Relay packets $\ell^v$, $\ell^t$ are
expressed in $\mathcal B$; two distinct odometric views identify the
relay-to-odometry yaw $\theta$ and hence the task vector $e$. The global
placement of the entire scene is an unobservable $SE(2)$ gauge.}
\label{fig:geometry}
\end{figure}
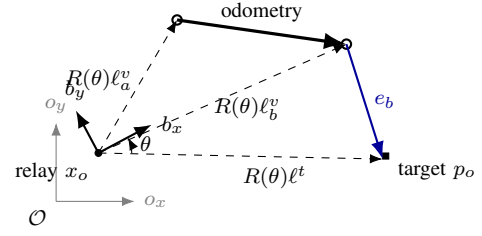


\section{Problem Formulation}
\label{sec:problem}

Fig.~\ref{fig:geometry} depicts the setting. Let $\mathcal W$, $\mathcal O$, and $\mathcal B$ denote the inaccessible
global frame, the vehicle's odometry frame, and the relay frame.
The vehicle, static relay, and static hidden target have global positions
$q_k, x, p \in \R^2$; the relay yaw is $\psi$, and the odometry frame has
unknown placement $(t_g,\gamma)\in SE(2)$ in $\mathcal W$, so the odometric
pose satisfies $q_k = t_g + \Rot(\gamma)s_k$.

At sampling instants the relay reports range and bearing packets to the
vehicle and to the target in its own frame,
\begin{equation}
\ell_k^v = \Rot(\psi)^{\!\top}(q_k - x), \qquad
\ell_k^t = \Rot(\psi)^{\!\top}(p - x),
\label{eq:packets}
\end{equation}
corrupted by independent range and bearing noise of standard deviations
$(\sigma_r,\sigma_\beta)$.
Defining the relay-to-odometry yaw $\theta = \psi - \gamma$ and the local
positions $x_o = \Rot(-\gamma)(x - t_g)$ and $p_o = \Rot(-\gamma)(p - t_g)$,
the packets obey the reduced model
\begin{equation}
s_k = x_o + \Rot(\theta)\,\ell_k^v, \qquad
p_o = x_o + \Rot(\theta)\,\ell_k^t ,
\label{eq:reduced}
\end{equation}
and the control-relevant task vector is
\begin{equation}
e_k \;=\; p_o - s_k \;=\; \Rot(\theta)\bigl(\ell_k^t - \ell_k^v\bigr),
\label{eq:task}
\end{equation}
which contains neither the global transform $(t_g,\gamma)$ nor the relay
position $x_o$: a single rotation $\theta$ is the entire calibration burden.

\begin{assumption}[Sensing and motion]
\label{ass:model}
(i) The vehicle measures body-frame odometry increments with independent
isotropic translation noise of per-increment standard deviation
$\sigma_s$, integrated into $\hat s_k$ with cumulative per-axis variance
$v_k$; (ii) the relay and target are static over the estimation window,
and relay-frame changes are piecewise constant in yaw with a dwell longer
than one window (Sec.~\ref{sec:closedloop}); (iii) packets are associated
and timestamped, and delivery may be delayed, jittered, reordered, or
dropped; (iv) no sensor measures any $\mathcal W$-frame quantity.
\end{assumption}

\section{Gauge Structure and Identifiability}
\label{sec:gauge}

\subsection{What no estimator can recover}

\begin{theorem}[$SE(2)$ gauge]
\label{thm:gauge}
Under Assumption~\ref{ass:model}, for every $c\in\R^2$ and
$\alpha\in S^1$ the transformation
$q_k' = c + \Rot(\alpha)q_k$, $x' = c + \Rot(\alpha)x$,
$p' = c + \Rot(\alpha)p$, $\psi' = \psi + \alpha$
leaves every packet \eqref{eq:packets} and every odometry increment
invariant. The indistinguishable set of the global state
$(q_{0:K},x,p,\psi)$ therefore contains a three-dimensional orbit, and no
causal estimator driven by the available signals can localize it.
Conversely, the quotient variables $(\theta, x_o, p_o)$ are a complete
set of invariants \emph{conditioned on the realized odometric trajectory}:
two global states sharing the same realized odometric trajectory
$s_{0:K}$ produce identical signal laws if and only if they share the same
quotient. Equivalently, the complete signal-level invariant tuple is
$\mathcal{I} = (s_{0:K}, \theta, x_o, p_o)$: the quotient describes
everything about the static scene that survives removal of the global
placement, while the trajectory must be carried as known context or
included in the invariant description. At any single view with known
odometric pose $s_k$ and packets $\ell_k^v,\ell_k^t$, the pair
$(\theta, e_k)$ is a control-sufficient statistic for seeking, but it is
equivalent to the full quotient only together with that known context:
$x_o$ and $p_o$ are recovered from $(\theta, e_k)$ via
\eqref{eq:reduced}, not from $(\theta,e_k)$ in isolation.
\end{theorem}

\begin{proof}
\emph{Step 1 (invariance).}
Direct substitution gives
\begin{equation*}
\begin{split}
\Rot(\psi')^{\!\top}(q_k'-x')
&= \Rot(\psi)^{\!\top}\Rot(\alpha)^{\!\top}\Rot(\alpha)(q_k-x)\\
&= \Rot(\psi)^{\!\top}(q_k-x),
\end{split}
\end{equation*}
and the same identity holds for $p - x$, so every packet is unchanged.
Odometry increments are body-frame quantities and do not involve the
global placement; the placement itself transforms as
$(t_g,\gamma)\mapsto(c+\Rot(\alpha)t_g,\,\gamma+\alpha)$, which absorbs
the orbit.

\emph{Step 2 (the quotient is invariant).}
Under the transformation of Step 1, $\theta = \psi - \gamma$ is unchanged
because both angles shift by $\alpha$, and $x_o$, $p_o$ are unchanged
because the shift of $t_g$ and the rotation by $\alpha$ cancel inside
$\Rot(-\gamma')(\cdot - t_g')$.

\emph{Step 3 (completeness, conditioned on the trajectory).}
Fix the realized odometric trajectory $s_{0:K}$ and the quotient
$(\theta,x_o,p_o)$. Then every packet is determined by the reduced model
\eqref{eq:reduced} evaluated along $s_{0:K}$, and every odometry statistic
by Assumption~\ref{ass:model} along the same trajectory, so the full
signal law is determined. The conditioning is not vacuous: two states with
the same static quotient but different odometric trajectories produce
different odometry increments (and different packets through
\eqref{eq:reduced}), so completeness holds relative to the tuple
$\mathcal{I}=(s_{0:K},\theta,x_o,p_o)$, not the static quotient alone.
Conversely, Theorem~\ref{thm:twoview} reconstructs the quotient from
noiseless signals along a fixed trajectory, so states with distinct
quotients are distinguishable. Together these give the equivalence.
\end{proof}

\begin{corollary}[Necessity of metric ego-motion]
\label{cor:necessity}
Without orientation-bearing metric ego-motion, the residual single-view
gauge of Corollary~\ref{cor:repeated} admits the alignments $\theta$ and
$\theta+\pi$, whose task vectors are $e$ and $-e$. Any deterministic
command has strictly positive inner product with at most one of them, so no
controller driven by relay packets alone can guarantee target-distance
descent in all indistinguishable worlds.
\end{corollary}

\subsection{Motion-induced self-calibration}

\begin{theorem}[Two-view identifiability, constructive]
\label{thm:twoview}
Two noiseless views with distinct odometric poses $s_a \neq s_b$ determine
the quotient uniquely and in closed form:
\begin{equation}
\begin{aligned}
\theta &= \angle(s_b - s_a) - \angle(\ell_b^v - \ell_a^v),\\
x_o &= s_a - \Rot(\theta)\,\ell_a^v,\qquad
p_o = x_o + \Rot(\theta)\,\ell^t .
\end{aligned}
\label{eq:twoview}
\end{equation}
\end{theorem}

\begin{proof}
\emph{Step 1 (yaw).}
Differencing the vehicle equations in \eqref{eq:reduced} yields
\begin{equation*}
s_b - s_a = \Rot(\theta)\bigl(\ell_b^v - \ell_a^v\bigr).
\end{equation*}
Both sides are nonzero and of equal norm, and exactly one planar rotation
maps one nonzero vector onto another of the same length; its angle is the
stated difference of orientations.

\emph{Step 2 (positions).}
With $\theta$ known, the first equation of \eqref{eq:reduced} gives $x_o$
at view $a$, the second gives $p_o$, and \eqref{eq:task} then yields every
task vector $e_k$.

\emph{Step 3 (uniqueness).}
Any second solution must produce the same rotation in Step 1, hence the
same $\theta$ modulo $2\pi$, and therefore the same $x_o$ and $p_o$ by
Step 2.
\end{proof}

\begin{corollary}[Repeated-pose gauge and rank transition]
\label{cor:repeated}
From a single pose, or any window of repeated poses, the one-parameter
family $(\theta+\alpha,\, s-\Rot(\theta+\alpha)\ell^v,\, \cdot)$ preserves
all packets: the local observability rank is four of five. One distinct
displacement raises it to five. Motion is therefore not merely helpful; it
is the exact mechanism that closes the gauge.
\end{corollary}

\begin{proposition}[Weighted multi-view registration]
\label{prop:registration}
For views $k=1,\dots,K$ with weights $w_k>0$, centered odometric poses
$a_k = \hat s_k - \bar s_w$, and centered relay vectors
$b_k = \ell_k^v - \bar\ell_w$, the estimator
\begin{equation}
\hat\theta = \operatorname{atan2}\Bigl(\textstyle\sum_k w_k\, b_k \times a_k,\;
\sum_k w_k\, b_k \cdot a_k\Bigr)
\label{eq:registration}
\end{equation}
recovers $\theta$ exactly from noiseless data whenever the spread
$S_v = \sum_k \lVert b_k\rVert^2$ is positive, reduces to
\eqref{eq:twoview} at $K=2$, and is the maximizer of the weighted
alignment $\sum_k w_k\, b_k^{\!\top}\Rot(\hat\theta)^{\!\top} a_k$.
It runs in $O(K)$ per update, measured at $3.6\,\mu$s for $K=64$ in our
implementation; the deployed $w_k$ is the inverse isotropicized noise
variance of view $k$, combining native range/tangential relay noise with
the pose-variance excess over the window's most-certain view.
\end{proposition}

\begin{proof}
\emph{Step 1 (exactness).}
Noiseless data satisfy $a_k = \Rot(\theta) b_k$, hence
\begin{equation*}
\begin{aligned}
\sum_k w_k\, b_k\times a_k &= \sin\theta \sum_k w_k\lVert b_k\rVert^2,\\
\sum_k w_k\, b_k\cdot a_k &= \cos\theta \sum_k w_k\lVert b_k\rVert^2,
\end{aligned}
\end{equation*}
so the $\operatorname{atan2}$ returns $\theta$ exactly whenever
$S_v > 0$.

\emph{Step 2 (optimality).}
The alignment objective has derivative in $\hat\theta$
\begin{equation*}
\sum_k w_k\, b_k^{\!\top}\Rot(\hat\theta)^{\!\top}S^{\!\top} a_k,
\end{equation*}
whose unique stationary maximum is the same $\operatorname{atan2}$
expression.

\emph{Step 3 (two-view case).}
At $K=2$ the centered vectors are proportional to the view differences,
and \eqref{eq:registration} collapses to the constructive formula
\eqref{eq:twoview}.
\end{proof}

\section{A Deployable Yaw Certificate under Correlated Odometry}
\label{sec:certificate}

Two uncertainty objects must not be conflated, and we claim each at its
full strength. The conditional information is exact given the poses; the
deployable certificate is exact to first order for the packet-noise and
integrated translational-odometry model of Assumption~\ref{ass:model},
including all cross-view correlations induced by shared increments.

\begin{lemma}[Conditional information]
\label{lem:conditional}
Conditional on the poses $s_{1:K}$, the Fisher information of
$(x_o,\theta)$ under native range and bearing noise has the yaw Schur
complement $I_\theta$; under isotropic Cartesianized noise of variance
$\sigma^2$ it specializes to $I_\theta = S_v/\sigma^2$. The spread $S_v$
is thus the exact geometric currency of calibration: zero on the
repeated-pose gauge and positive otherwise.
\end{lemma}

\begin{theorem}[Correlated-odometry yaw variance, exact to first order]
\label{thm:certificate}
Let the odometric poses be $\hat s_k = s_k + e_k$, where the errors stem
from integrated independent isotropic translation increments with
cumulative per-axis variances $v_k$, gauge fixed at the least uncertain
view with variance $v_{\min}$. Then
\begin{equation}
\operatorname{Cov}(e_i, e_j) \;=\; \bigl(\min(v_i,v_j) - v_{\min}\bigr) I_2 ,
\label{eq:cov}
\end{equation}
and the first-order variance of \eqref{eq:registration} is
\begin{equation}
\operatorname{var}(\hat\theta)
= \underbrace{\frac{\sum_k w_k^2 \lVert b_k\rVert^2
\sigma_{\perp,k}^2}{W^2}}_{\text{packet noise}}
+ \underbrace{\sum_{m=2}^{K} \Delta v_m\,
\lVert G_m \rVert^2}_{\text{corr.\ odometry}},
\label{eq:varfull}
\end{equation}
where $W = \sum_k w_k\lVert b_k\rVert^2$,
$\sigma_{\perp,k}^2$ is the native range-bearing packet covariance
$\operatorname{diag}(\sigma_r^2,\sigma_{t,k}^2)$, $\sigma_{t,k}=r_k\sigma_\beta$,
projected onto the direction perpendicular to $b_k$ (App.~\ref{app:fisher}
gives the general anisotropic case),
$\Delta v_m = v_m - v_{m-1} \ge 0$, and
$G_m = \sum_{k\ge m} g_k$ is the suffix sum of the pose sensitivities
$g_k = \partial\hat\theta/\partial e_k$. The odometry term is an exact
rewriting of the full quadratic form
$\sum_{i,j} g_i^{\!\top}\operatorname{Cov}(e_i,e_j)\, g_j$
and is computable in $O(K)$.
\end{theorem}

\begin{proof}
\emph{Step 1 (covariance structure).}
Pose errors are partial sums of independent increments. The covariance of
two partial sums is the variance of their shared prefix,
$\min(v_i,v_j)\,I_2$, and subtracting the gauge reference removes the
common $v_{\min}$, giving \eqref{eq:cov}.

\emph{Step 2 (packet-noise term).}
Linearizing \eqref{eq:registration} in a packet perturbation $\nu_k$ of
$\ell_k^v$ gives the contribution
\begin{equation*}
d\hat\theta \;=\; \frac{1}{W}\sum_k w_k\,(\nu_k \times b_k),
\end{equation*}
because only components of $\nu_k$ perpendicular to $b_k$ survive the
cross product; taking variances yields the first term of
\eqref{eq:varfull}.

\emph{Step 3 (pose sensitivities and automatic centering).}
Linearizing in a pose perturbation $e_k$ gives $d\hat\theta = g_k^{\!\top}
e_k$ with, writing $c_x = \sum_k w_k\, b_k\cdot a_k$ and
$c_y = \sum_k w_k\, b_k\times a_k$ for the weighted cosine- and
sine-alignment sums of \eqref{eq:registration} evaluated at the
linearization point,
\begin{equation*}
g_k = \frac{w_k}{C^2}
\begin{pmatrix} -c_x b_{k,y} - c_y b_{k,x} \\[2pt]
\phantom{-}c_x b_{k,x} - c_y b_{k,y} \end{pmatrix},
\qquad C^2 = c_x^2 + c_y^2 .
\end{equation*}
Because the $b_k$ are centered, $\sum_k g_k = 0$: the estimator is
automatically invariant to a common pose offset, so mean removal of the
pose errors requires no explicit correction.

\emph{Step 4 (suffix form).}
Substituting \eqref{eq:cov} into the quadratic form and expanding
$\min(v_i,v_j)-v_{\min} = \sum_{m=2}^{\min(i,j)}\Delta v_m$
exchanges the order of summation:
\begin{equation*}
\begin{split}
\sum_{i,j} g_i^{\!\top}\operatorname{Cov}(e_i,e_j)\, g_j
&= \sum_{m=2}^{K} \Delta v_m
\Bigl\lVert \sum_{k \ge m} g_k \Bigr\rVert^2 \\
&= \sum_{m=2}^{K} \Delta v_m \lVert G_m\rVert^2 ,
\end{split}
\end{equation*}
which is the stated $O(K)$ evaluation.
\end{proof}

\begin{remark}[Empirical calibration of the deployable certificate]
\label{rem:coverage}
For a two-sided level $1-\alpha$, let $z_{1-\alpha/2}$ denote the standard
normal critical value; in particular, $z_{0.975}=1.96$ for 95\% coverage.
The interval
$\hat\theta \pm z_{1-\alpha/2}\sqrt{\operatorname{var}(\hat\theta)}$
is empirically calibrated for $\theta$ over the
tested operating envelope: across ten odometry conditions spanning
translation noise of $0.25$ to $5$\,cm per increment and window sizes of
$8$ to $16$ views, with $2{,}000$ realizations per condition
($20{,}000$ trials in total), per-condition empirical-to-predicted
variance ratios lie between $0.88$ and $1.07$, the pooled ratio is
$0.91$, and the realized 95\% coverage lies between $94.0\%$ and
$95.6\%$ per condition with a pooled value of $94.8\%$.
\end{remark}

\begin{remark}[Why this matters downstream]
Every closed-loop guarantee in Sec.~\ref{sec:closedloop} consumes
\eqref{eq:varfull} rather than the conditional information: the supervisor
certifies, refuses, and detects change against an uncertainty that is
honest about odometry. The contrast is not academic. A sequential EKF
given the same packets, the same two-view initialization, and a matched
bias state certifies its own wrong yaw after unlucky initializations
precisely because its self-reported covariance lacks this calibration.
\end{remark}

\section{Certified Acquisition, Seeking, and Recovery}
\label{sec:closedloop}

Building on spread-supervised seeking in the globally anchored
model~\cite{bagla2026closedloop}, the hybrid policy below replaces the
spread-only trigger with the correlated-odometry certificate and adds
principled refusal and certified relay-frame change adoption. Its modes
are \textsc{excite}, \textsc{seek}, and \textsc{maintain}: a
constant-curvature arc until the certificate clears
$z_{0.975}\sqrt{\operatorname{var}(\hat\theta)} \le \bar\delta$ with
$\bar\delta = 10^\circ$ on consecutive certified windows, then saturated
seeking on the filtered task vector, with hysteresis and dwell.

\begin{proposition}[Finite-window certification and refusal condition]
\label{prop:acquisition}
For a rolling window $\mathcal K$ define
\begin{equation*}
 B(\mathcal K)=
 \frac{\sum_{k\in\mathcal K}w_k^2\lVert b_k\rVert^2
 \sigma_{\perp,k}^2}{W^2}
 +\sum_{m\in\mathcal K\setminus\{1\}}\Delta v_m\lVert G_m\rVert^2 .
\end{equation*}
Suppose a constant-curvature acquisition arc supplies admissible windows
with $S_v>0$ and correlation
$\rho=\sqrt{c_x^2+c_y^2}/\sum_{k\in\mathcal K}w_k\lVert a_k\rVert\lVert b_k\rVert
\in[0,1]$ (with $c_x,c_y$ as in Theorem~\ref{thm:certificate}'s proof)
above the supervisor's gate $\rho_{\min}$; $\rho=1$ exactly when every
view's implied rotation agrees. \emph{If} $L$ consecutive
windows obey $B(\mathcal K)\le(\bar\delta/z_{1-\alpha/2})^2$, \emph{then}
the supervisor certifies after at most those $L$ window updates; whether
such windows are realized on a given arc is a property of the realized
odometry noise and arc geometry, and guaranteeing it once the window
saturates would require an additional stationarity or ergodicity
assumption on the acquisition process that we do not make (see
Appendix~\ref{app:acquisition} for the full argument). The refusal branch
carries no such conditional: if no admissible rolling window obeys this
inequality, the supervisor remains in \textup{\textsc{excite}} and refuses
to release an uncertified yaw, unconditionally. Thus acquisition is
governed by a finite-window signal-to-odometry condition, not by mission
duration.
\end{proposition}

\begin{proof}[Proof sketch]
\emph{Step 1.} Distinct poses on the acquisition arc give $S_v>0$, hence a
valid registration and finite $W$ by Proposition~\ref{prop:registration}.

\emph{Step 2.} Theorem~\ref{thm:certificate} gives
$\operatorname{var}(\hat\theta)=B(\mathcal K)$ for every rolling window;
the suffix term already contains the correlations within that window, so
no accumulation assumption across successive windows is required.

\emph{Step 3.} The supervisor leaves \textup{\textsc{excite}} only after $L$
consecutive windows satisfy its confidence and correlation gates. The
first condition in the statement therefore implies finite certification.
If the variance inequality is infeasible over all admissible windows, the
same gate prevents transition by construction. The fixed window bounds
odometry exposure, so waiting longer cannot substitute for informative
geometry.
\end{proof}

\begin{remark}[Validation]
\label{rem:acquisition-validation}
Both branches are observed exactly in the clean campaign: certification in
$200/200$ trials at nominal odometry noise and $188/200$ at $1$\,cm per
increment; principled refusal at $2$ to $5$\,cm, where no tested
estimator, including a fixed-lag smoother roughly $50\times$ more
expensive, extends the boundary. The tested boundary is governed primarily
by finite-window signal relative to odometry uncertainty, consistent with
the finite-window bound (Sec.~\ref{sec:experiments}).
\end{remark}

\begin{theorem}[Certified seeking: contraction and non-accumulating error]
\label{thm:seeking}
Let the task kinematics be $\dot e = -u$ with the seeking law
$u = k_p\,\hat e$, where the filtered estimate satisfies
$\hat e = \Rot(\tilde\theta)\,e + \eta$ with certified yaw error
$|\tilde\theta| \le \bar\delta < \pi/2$ and bounded estimate disturbance
$\norm{\eta} \le \bar\eta$. Then along closed-loop trajectories
\begin{equation}
\frac{d}{dt}\,\norm{e}
\;\le\; -\,k_p \cos\bar\delta \,\norm{e} \;+\; k_p\,\bar\eta ,
\label{eq:contraction}
\end{equation}
so $\norm{e}$ converges exponentially at rate $k_p\cos\bar\delta$ to the
ultimate ball of radius $\bar\eta/\cos\bar\delta$. The radius depends on
packet noise, filter lag, and odometry bias through $\bar\eta$ only, and
is independent of the mission horizon: the task vector is re-measured by
every packet through \eqref{eq:task}, so integrated odometry error enters
$\bar\eta$ only through the finite calibration window and never
accumulates.
\end{theorem}

\begin{proof}
\emph{Step 1 (radial derivative).}
With $V = \tfrac12\norm{e}^2$,
\begin{equation*}
\dot V = -\,e^{\!\top} u
= -\,k_p\, e^{\!\top}\Rot(\tilde\theta)\,e \;-\; k_p\, e^{\!\top}\eta .
\end{equation*}

\emph{Step 2 (rotation bound).}
The symmetric part of $\Rot(\tilde\theta)$ is $\cos\tilde\theta\, I_2$, so
\begin{equation*}
e^{\!\top}\Rot(\tilde\theta)e = \cos\tilde\theta\,\norm{e}^2
\ge \cos\bar\delta\,\norm{e}^2, \qquad |\tilde\theta|\le\bar\delta<\pi/2.
\end{equation*}

\emph{Step 3 (comparison).}
Bounding the disturbance term by $k_p\norm{e}\bar\eta$ and dividing by
$\norm{e}$ yields \eqref{eq:contraction}; the comparison lemma gives
exponential convergence to $\bar\eta/\cos\bar\delta$.

\emph{Step 4 (horizon independence).}
$\eta$ collects packet noise, the filter's first-order lag, and the
effect of odometry bias accumulated over at most one calibration window
of fixed length; none of these grows with mission time, so $\bar\eta$,
and hence the ultimate radius, is horizon independent. Saturation of $u$
rescales the rate but preserves the sign of $\dot V$, and the certified
event $|\tilde\theta|\le\bar\delta$ is calibrated by
Remark~\ref{rem:coverage}. More generally, on any finite horizon where
$\Pr\{|\tilde\theta|\le\bar\delta\}\ge1-\alpha_\theta$ and
$\Pr\{\sup_t\lVert\eta(t)\rVert\le\bar\eta\}\ge1-\alpha_\eta$, the union
bound gives the contraction conclusion with probability at least
$1-\alpha_\theta-\alpha_\eta$.
\end{proof}

\begin{lemma}[Unicycle realization of the virtual seeking law]
\label{lem:unicycle}
The controller realized in code (Supervisor.hpp) is a saturated unicycle
law, not the virtual holonomic law $\dot e=-u$ assumed above: forward speed
and turn rate are commanded on separate channels,
$v=\min(v_{\max},k_p\norm{\hat e})\max(0,\cos\alpha)$ and
$\omega=k_\omega\alpha$, where $\alpha$ is the heading error between the
vehicle's heading and $\angle\hat e$. Writing $u=k_p\hat e$ as in
Theorem~\ref{thm:seeking}, the realized task dynamics satisfy
$\dot e = -u + d(t)$ with
$d(t) = u - v\,\Rot(-\alpha)\,\hat e/\norm{\hat e}$. Whenever the forward
channel is unsaturated ($k_p\norm{\hat e}\le v_{\max}$) and $\cos\alpha\ge
0$, $\norm{d(t)} = k_p\norm{\hat e}\,\lvert\sin\alpha\rvert$, which
vanishes as $\alpha\to0$; otherwise (saturation or $\cos\alpha<0$),
$\norm{d(t)}\le k_p\norm{\hat e}+v_{\max}$, a fixed bound whenever
$\norm{\hat e}$ is bounded. Consequently the realized dynamics are the
virtual dynamics with an enlarged disturbance: with any bound
$\norm{d(t)}\le\bar d$ over the regime of interest,
\begin{equation*}
\dot e = -k_p\bigl[\Rot(\tilde\theta)e + \eta_{\rm eff}\bigr],
\quad
\eta_{\rm eff}=\eta-\frac{d}{k_p},
\quad
\norm{\eta_{\rm eff}}\le\bar\eta+\frac{\bar d}{k_p},
\end{equation*}
and Theorem~\ref{thm:seeking} applies verbatim with $\bar\eta$ replaced by
$\bar\eta+\bar d/k_p$: its proof uses only a bound on the disturbance
norm, not its source. We claim no rigorous gain conditions for the full
heading-error cascade; the realization residual is bounded, vanishes as
the heading error converges, and its effect is validated empirically in
closed loop, including on the physics-based differential-drive vehicle.
\end{lemma}

\begin{proof}[Proof sketch]
\emph{Step 1 (kinematics).}
The vehicle drives forward along its own heading, so
\begin{equation*}
\dot s = v\begin{pmatrix}\cos(\angle\hat e-\alpha)\\ \sin(\angle\hat e-\alpha)\end{pmatrix}
= v\,\Rot(-\alpha)\,\frac{\hat e}{\norm{\hat e}},
\end{equation*}
and since the target and relay are static, $\dot e = -\dot s$. Adding and
subtracting $u=k_p\hat e$ gives
\begin{equation*}
\dot e = -u + \Bigl(u - v\,\Rot(-\alpha)\frac{\hat e}{\norm{\hat e}}\Bigr)
= -u + d(t).
\end{equation*}

\emph{Step 2 (unsaturated bound).}
When unsaturated, $v=k_p\norm{\hat e}\cos\alpha$: the realized velocity is
the \emph{projection} of the ideal velocity onto the heading direction, so
the two vectors need not have equal magnitude, and the residual is the
perpendicular component. Explicitly, with $\hat g=\hat e/\norm{\hat e}$
and coordinates in which $\hat g=(1,0)^{\!\top}$, so that
$\Rot(-\alpha)\hat g=(\cos\alpha,\,-\sin\alpha)^{\!\top}$,
\begin{equation*}
d = k_p\norm{\hat e}
\begin{pmatrix}1-\cos^2\alpha\\ \cos\alpha\sin\alpha\end{pmatrix}
= k_p\norm{\hat e}\,\sin\alpha
\begin{pmatrix}\sin\alpha\\ \cos\alpha\end{pmatrix},
\end{equation*}
hence $\norm{d(t)} = k_p\norm{\hat e}\,\lvert\sin\alpha\rvert$, which
vanishes as $\alpha\to0$.

\emph{Step 3 (saturated bound).}
Under saturation or $\cos\alpha<0$,
$\norm{v\,\Rot(-\alpha)\hat e/\norm{\hat e}}\le v_{\max}$, so the triangle
inequality gives the fixed bound
\begin{equation*}
\norm{d(t)} \;\le\; k_p\norm{\hat e} + v_{\max}.
\end{equation*}

\emph{Step 4 (heading loop, scope of the claim).}
The commanded turn rate $\omega=k_\omega\alpha$ drives
$\dot\alpha=-k_\omega\alpha+\rho(t)$, where $\rho$ collects the rotation
of the commanded direction $\angle\hat e$ induced by the vehicle's own
motion and by filter updates. We do not develop gain conditions for the
full cascade; the lemma asserts only the two regime bounds of Steps 2--3:
$d$ is bounded whenever $\norm{\hat e}$ is bounded, and vanishes with the
heading error. So the virtual-controller contraction of
Theorem~\ref{thm:seeking} applies with the enlarged disturbance
$\eta_{\rm eff}$, and the closed-loop experiments (including the
physics-based differential-drive validation) confirm the realization
residual does not alter the observed contraction.
\end{proof}

\begin{remark}[Validation]
\label{rem:seeking-validation}
Over $600$\,s the closed loop holds $0.064$\,m median station error while
dead reckoning drifts to a median of $27$\,m, and it is statistically
indistinguishable from the known-yaw oracle at $0.065$\,m: the certificate
pipeline extracts essentially all available information
(Sec.~\ref{sec:experiments}).
\end{remark}

\begin{proposition}[Recovery under piecewise-constant relay yaw]
\label{prop:recovery}
Let the relay yaw be piecewise constant with steps of size $\Delta$ and
dwell longer than one window turnover, and let adoption require $L$
consecutive certified windows whose estimates are $\chi^2$ inconsistent
with the control path at level $\alpha_\chi$, persistence through a full
window turnover, and a minimum inter-adoption interval. Assume the
first-order standardized yaw residuals are Gaussian and that $c$ disjoint
window turnovers contributing to an adoption are independent. Then:
(i) with no step, the probability of a false adoption is bounded by
$\alpha_\chi^{\,c}$;
(ii) after a step with $|\Delta|$ large relative to the certified band,
every post-step certified window is inconsistent with the stale control
path with probability approaching one, so adoption occurs within $L$
windows plus one turnover once motion supplies certified views;
(iii) a step needs no adoption whenever its combined error satisfies
$|\tilde\theta_{\rm old}|+|\Delta|\le\delta_c<\pi/2$: the loop remains
contracting with ultimate radius $\bar\eta/\cos\delta_c$.
\end{proposition}

\begin{proof}[Proof sketch]
\emph{Step 1.} Under no step, a certified window estimate differs from
the control path by a zero-mean error whose standardized square exceeds
the $\chi^2_1$ threshold with probability $\alpha_\chi$ under the stated
first-order Gaussian model. Requiring persistence across $c$ independent,
disjoint turnovers multiplies the bound.

\emph{Step 2.} After a large step, the window estimate concentrates near
$\theta + \Delta$ while the control path holds $\theta$; the
standardized discrepancy grows as $\Delta^2$ against the summed
variances, so the inconsistency test passes every window once windows are
certified, and the persistence counter reaches $L$ in $L$ packets plus
the turnover.

\emph{Step 3.} The triangle inequality bounds the post-step yaw error by
$|\tilde\theta_{\rm old}|+|\Delta|$. When this is at most
$\delta_c<\pi/2$, Theorem~\ref{thm:seeking} applies with $\delta_c$, so no
adoption is required for task success.
\end{proof}

\begin{remark}[Validation]
\label{rem:recovery-validation}
In the clean campaign, task success is retained in $150/150$ trials at
every tested $\Delta \in \{20^\circ,40^\circ,80^\circ\}$ and mission
phase; recovery is threshold sharp, with $80^\circ$ steps recalibrating
in $98$ to $100\%$ of required cases once motion supplies views and
$20^\circ$ steps often requiring no action; and zero false adoptions
occurred in $200$ undisturbed trials per policy, with a Wilson 95\% upper
bound of $1.9\%$ (Sec.~\ref{sec:experiments}).
\end{remark}

\begin{remark}[Scope]
Proposition~\ref{prop:recovery} is a robustness extension; the static-relay
results of Theorems~\ref{thm:gauge} to \ref{thm:certificate} do not depend
on it. Recovery requires motion by Corollary~\ref{cor:necessity}, and the
station-keeping case delays recalibration until the task itself demands
motion. That is the correct behavior: at $e \approx 0$, yaw error is
provably task irrelevant by \eqref{eq:task}.
\end{remark}

\section{Experimental Validation}
\label{sec:experiments}

\begin{figure*}[t]
\centering
\includegraphics[width=0.92\textwidth]{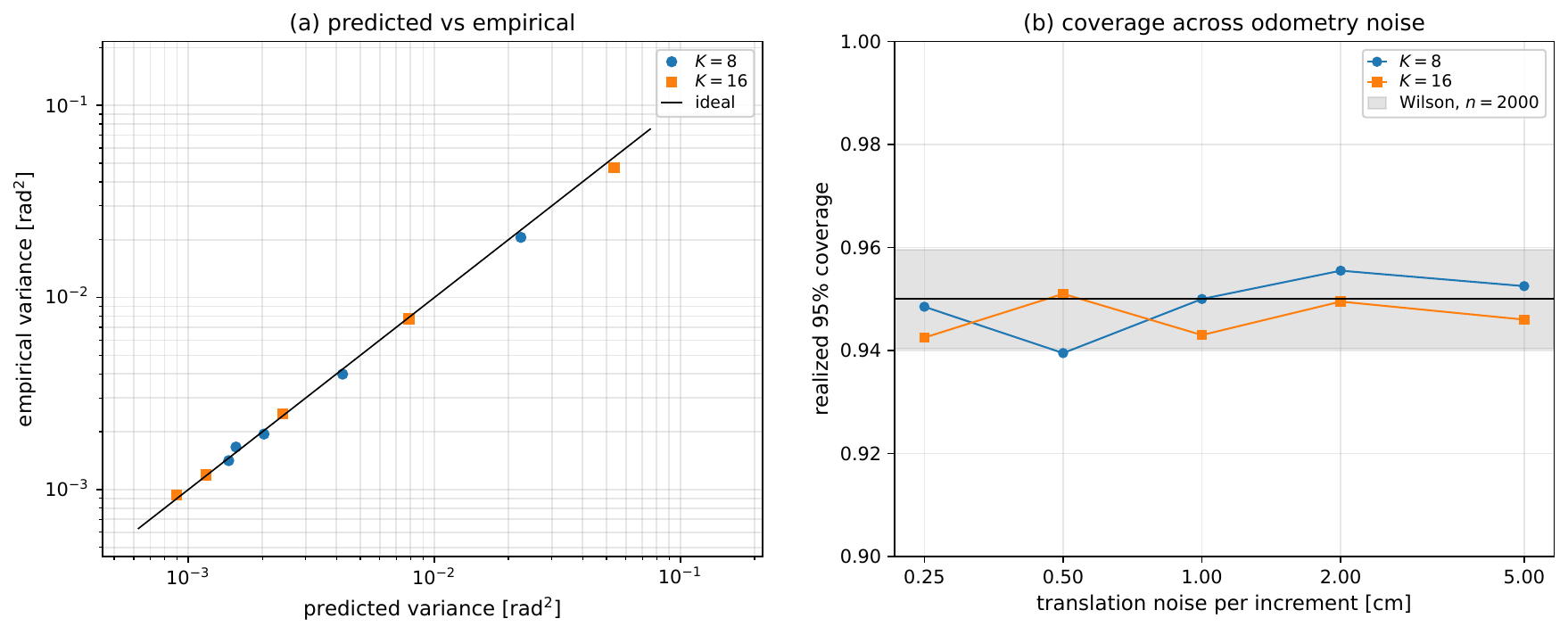}
\caption{Certificate calibration on the clean campaign.
(a)~Predicted versus empirical yaw variance across all ten odometry
conditions. (b)~Realized 95\% coverage per condition against the nominal
level with the finite-sample Wilson band for $2{,}000$ trials per point.}
\label{fig:certificate}
\end{figure*}

All quantitative claims in this section come from a single
provenance-locked campaign: every manifest records the producing commit
with a clean tree, and all artifacts are checksummed. Randomized
geometries draw target range, relay range, relay yaw, initial heading, and
frame placement per trial; statistics are medians with Wilson or bootstrap
$95\%$ intervals; no trial is discarded.

\textbf{Setup and metrics.} The unicycle simulation runs at $100$\,Hz with
relay packets at $20$\,Hz for $150$\,s, except for the $600$\,s drift study.
Target and relay ranges are uniform on $[6,15]$ and $[4,12]$\,m, all
orientations are uniform on $[-\pi,\pi)$, and the entities remain at least
$1$\,m apart. The registration retains at most $64$ packets or $4$\,s and
certification requires three consecutive windows with at least eight
packets, correlation $\rho\ge0.5$ (Proposition~\ref{prop:acquisition}),
and a 95\% yaw half-width below $10^\circ$. Nominal noise is $0.10$\,m and $1^\circ$ per relay channel,
$0.5$\,cm translational and $0.001$\,rad heading noise per odometry
increment, with body bias $(1,-0.5)$\,cm/s. Success means reaching
$0.25$\,m and then remaining within $0.35$\,m for $10$\,s; station RMSE is
computed over the final $30$\,s. Runtimes are Release-mode C++ update costs
on an Intel Core Ultra 9 275HX.

Fig.~\ref{fig:trajectories} shows representative closed-loop missions:
the exploratory arc, certified release into seeking, station keeping, and,
under a mid-transit relay rotation, the detection detour and recovery. The
dead-reckoned track is overlaid to make the operating premise concrete:
the pose estimate the vehicle would trust under dead reckoning visibly
diverges while the task loop remains locked to the target.

\begin{figure}[t]
\centering
\includegraphics[width=\linewidth]{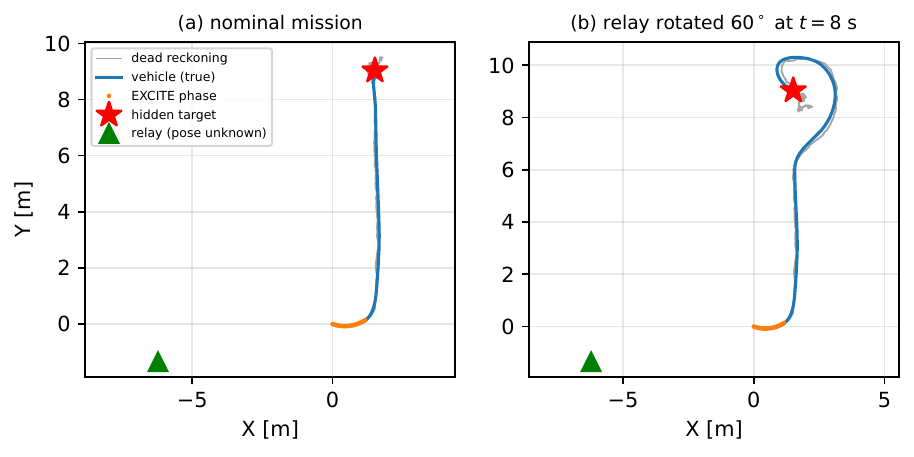}
\caption{Representative missions (illustrative runs at the campaign
commit; all statistics come from the randomized campaign).
(a)~Exploratory arc, certified seeking, station keeping.
(b)~A $60^\circ$ relay rotation at $t=8$\,s forces re-detection: the
detour and recovery are visible in the track.}
\label{fig:trajectories}
\end{figure}

\textbf{Baselines.} Across $200$ paired randomized geometries the proposed
estimator succeeds in $200/200$ trials with $0.064$\,m median station
error, statistically indistinguishable from the known-yaw oracle at
$0.065$\,m: the paired oracle-minus-proposed median is $+0.001$\,m with
bootstrap 95\% interval $[-0.004,+0.005]$\,m. A fair sequential EKF with the same packets, the same two-view
initialization, and a matched body-frame bias state reaches $0.087$\,m;
its paired penalty of $+0.019$\,m has bootstrap interval
$[+0.015,+0.028]$. Endpoint-only registration succeeds in $194/200$ with
roughly twice the acquisition time, quantifying the value of the full
window. A fixed-lag smoother roughly $50\times$ more expensive is
statistically indistinguishable from the proposed estimator in the tested
regime and extends no operating boundary: the
closed-form estimator is information sufficient. Without exploratory
motion, $0/200$ trials certify, the empirical face of
Corollary~\ref{cor:necessity}.

\textbf{Certificate coverage.} Across ten odometry conditions with
$2{,}000$ realizations each ($20{,}000$ trials), per-condition
empirical-to-predicted variance ratios lie between $0.88$ and $1.07$ with
a pooled ratio of $0.91$, and realized coverage of the $95\%$ interval is
$94.0$ to $95.6\%$ per condition and $94.8\%$ pooled
(Remark~\ref{rem:coverage}, Fig.~\ref{fig:certificate}). This regime spans
the $8$--$16$-view windows enforced by the supervisor's minimum-window
gate (Sec.~\ref{sec:closedloop}); a weakly conditioned two-view window can
undercover substantially, since a $0.25$\,m chord case overstates
precision by $74\%$ and covers only $91.6\%$ at the nominal $95\%$ level
(Table~\ref{tab:conditioning}, App.~\ref{app:tables}). Such windows are
excluded from certification by construction, not merely by favorable
testing.

\textbf{Covariance ablation.} To isolate the contribution of modeling
cross-view odometry correlation in closed form, Fig.~\ref{fig:covablation}
reruns the same odometry-noise grid, with an independent Monte Carlo seed,
under three certificate variance models: the proposed correlated
random-walk certificate (\emph{full}, Theorem~\ref{thm:certificate}), a
naive model that treats each view's pose error as independent of the
others (\emph{diag}), and a model that drops the pose term entirely, as if
odometry were exact (\emph{packet-only}). Pooled over the grid,
\emph{full} holds ratio $0.95$ and $95.0\%$ coverage, matching
Remark~\ref{rem:coverage}'s $0.91$ and $94.8\%$ up to the Monte Carlo
sampling variation expected from the different seed; \emph{diag} degrades
to ratio $3.64$ and $82.1\%$ coverage; \emph{packet-only} degrades further
to ratio $9.22$ and $69.3\%$ coverage. False certification, meaning a
window whose reported half-width passes the $10^\circ$ release gate
($z_{0.975}\sqrt{\operatorname{var}(\hat\theta)}\le10^\circ$) yet whose
actual yaw error still exceeds $10^\circ$, occurs in $0.3\%$, $5.1\%$, and
$6.9\%$ of trials respectively: a narrower, gate-conditioned statistic
than coverage, since most miscovered \emph{diag}/\emph{packet-only}
windows also fail the release gate outright. Closed-loop success stays
$200/200$ for all three at the nominal condition: the wrong models are
dishonest about their own confidence well before that dishonesty costs a
mission. Per-cell,
\emph{full} never leaves a $0.92$--$1.03$ ratio band anywhere in the grid,
while both ablations worsen monotonically with odometry noise and window
size (worst cells: \emph{diag} ratio $6.71$/coverage $52.2\%$/false-cert
$27.9\%$; \emph{packet-only} ratio $57.47$/coverage $22.0\%$/false-cert
$39.5\%$), consistent with more views accumulating more correlated error
that only the full model accounts for.

\begin{figure}[t]
\centering
\includegraphics[width=\linewidth]{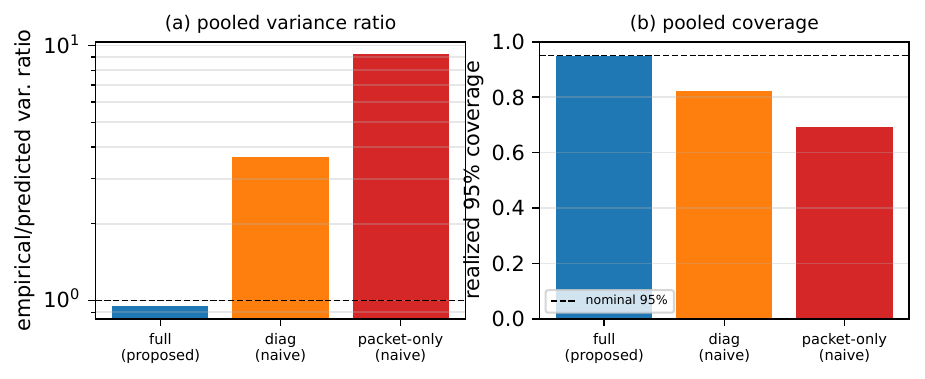}
\caption{Covariance ablation, pooled over the ten-condition odometry grid
($2{,}000$ realizations/cell). (a)~Empirical-to-predicted yaw-variance
ratio (log scale; dashed line is exact). (b)~Realized $95\%$ coverage
(dashed line is nominal). Only the proposed certificate (\emph{full})
stays calibrated; both naive ablations undercover.}
\label{fig:covablation}
\end{figure}

\textbf{Robustness.} Fig.~\ref{fig:robustness} summarizes the
one-factor-at-a-time sweeps: $200/200$ success at
every relay noise level up to
$0.4$\,m range and $4^\circ$ bearing noise, with median error below
$0.17$\,m at the harshest setting. Communication stress preserves
$200/200$ success through $50\%$ dropout, $0.5$\,s fixed delay, $0.2$\,s
jitter, and $20\%$ gross packet outliers; dropout, delay, and jitter
leave accuracy essentially flat, and only gross outliers degrade it,
gracefully reaching $0.204$\,m in the hardest case.

\begin{figure}[t]
\centering
\includegraphics[width=\linewidth]{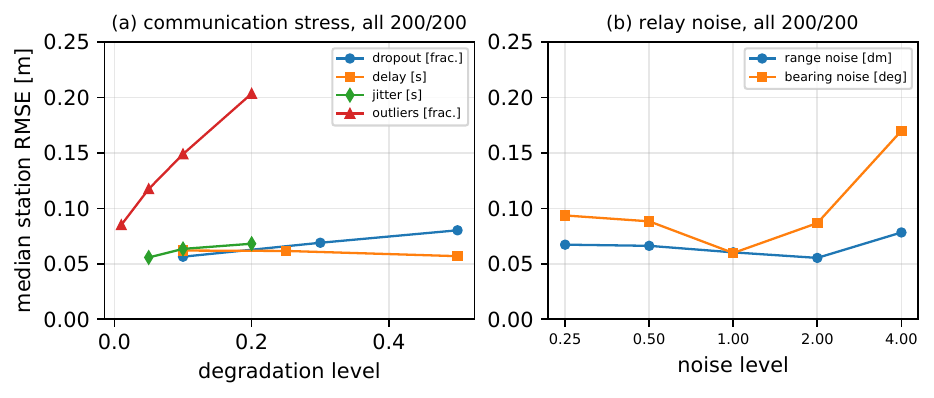}
\caption{Robustness sweeps at $200$ trials per cell, $100\%$ success in
every cell shown. (a)~Communication stress: dropout, delay, and jitter
leave station accuracy essentially flat; gross packet outliers degrade it
gracefully. (b)~Relay range and bearing noise sweeps.}
\label{fig:robustness}
\end{figure} Odometry sweeps confirm the
operating boundary of Proposition~\ref{prop:acquisition}
(Remark~\ref{rem:acquisition-validation}). Bias to $5$\,cm/s, scale error
to $10\%$, and heading drift to $0.5^\circ$/s all retain $200/200$. The
analytical certificate
(Theorem~\ref{thm:certificate}) covers translational random-walk odometry
under Assumption~\ref{ass:model}; bias, scale error, and heading drift lie
outside that stochastic model and are reported here only as empirical
robustness results, not as certified operating conditions.

\textbf{Physics-based ROS~2/Gazebo validation.} The identical estimator,
certificate, and controller also drive a physics-based differential-drive
vehicle in Gazebo: wheel joints with contact friction and chassis
inertia, wheel odometry from the simulator's own drive-train integration
rather than ground truth, and asynchronous packet delivery through the
same ROS~2 interfaces intended for deployment. Table~\ref{tab:gazebo}
reports $30$ independent launches per scenario with seeded channel
realizations; Gazebo's physics integration is not bitwise deterministic
run to run, so only the aggregate statistics below are interpreted.
Nominal station RMSE rises modestly to $0.077$\,m from the offline
$0.065$\,m, consistent with wheel slip and inertia the offline unicycle
omits, and success is $30/30$ in every scenario. Under a combined
communication-degradation scenario ($30\%$ dropout, $0.2$\,s delay,
$0.1$\,s jitter, and $5\%$ outliers applied simultaneously, each factor
milder than the offline one-factor sweeps but combined rather than
isolated), station RMSE rises to $0.164$\,m median while success remains
$30/30$. Injecting the $60^\circ$ relay rotation mid-transit, four to
five meters from the target rather than after arrival, the supervisor
adopts the new registration at a median $7.3$\,s, essentially identical
across seeds because the disturbance vastly exceeds the certificate's
noise floor and the same persistence gate that produced zero false
adoptions above also sets a fixed reaction floor here; the vehicle
recovers the task at a median $10.3$\,s and every trial succeeds. Every
Gazebo run is provenance-locked to a single clean commit under the same
result contract as the offline campaign, and the batch runner, ROS~2
packages, and world file ship with the paper.

\begin{table}[t]
\caption{Gazebo Validation}
\label{tab:gazebo}
\centering
\begingroup\scriptsize\itshape
Physics-based vehicle, 30 independent launches per scenario; bracketed
values are interquartile ranges. Station RMSE is over the final 30\,s.
Reach is the time to first satisfy the success threshold of
Sec.~\ref{sec:experiments}.
\par\endgroup
\vspace{2pt}
\footnotesize
\setlength{\tabcolsep}{3.0pt}
\begin{tabular}{@{}lccc@{}}
\toprule
Scenario & Success & Station RMSE [m] & Reach [s]\\
\midrule
Nominal & $30/30$ & $0.077\;[0.051,0.094]$ & $10.9\;[10.8,11.0]$\\
Comm.\ stress & $30/30$ & $0.164\;[0.149,0.198]$ & $11.4\;[11.2,11.7]$\\
Mid-transit $60^\circ$ step & $30/30$ & $0.043\;[0.031,0.070]$ & $17.1\;[17.0,17.3]$\\
\bottomrule
\end{tabular}
\end{table}

\begin{figure*}[t]
\centering
\includegraphics[width=\textwidth]{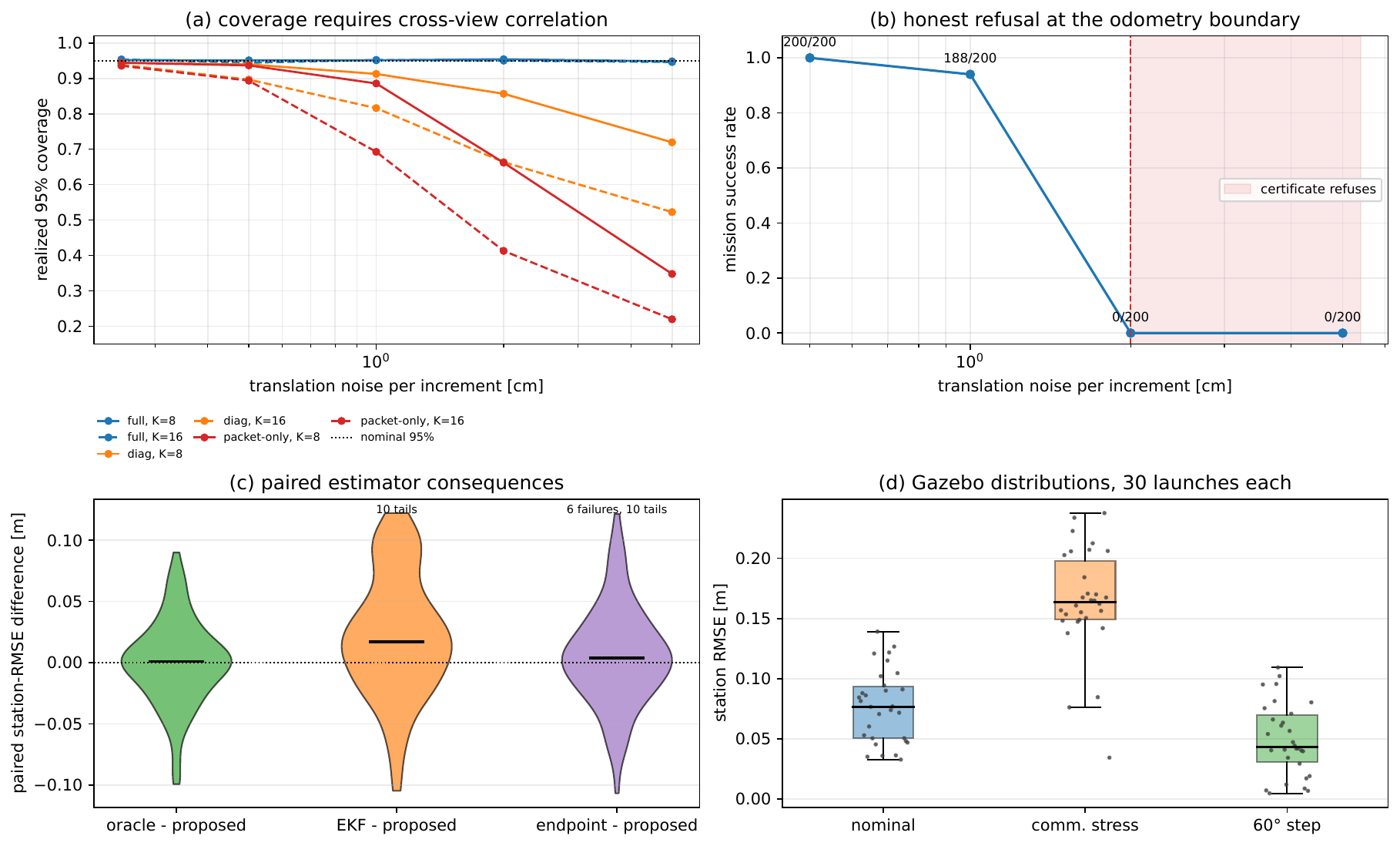}
\caption{Operating envelope across the reported simulation studies. (a)~Per-condition
coverage of the nominal $95\%$ yaw interval for the full correlated model,
the diagonal approximation, and packet noise alone ($2000$ trials per
cell). (b)~Closed-loop success versus translational odometry noise per
integration step ($200$ trials per level); the dashed line marks the first
tested level at which the certificate consistently refuses release.
(c)~Central paired station-RMSE differences over mutually successful trials
on the same $200$ randomized geometries; positive values favor the proposed
estimator, while failures and off-scale tails are annotated separately.
(d)~Physics-based Gazebo
station-RMSE distributions over $30$ launches per scenario. Boxes show the
interquartile range and every dot is one launch.}
\label{fig:operating-envelope}
\end{figure*}

Figure~\ref{fig:operating-envelope} exposes details hidden by pooled
summaries. Only the full covariance retains approximately nominal coverage
across view count and increment-noise level; dropping cross-view terms or
odometry uncertainty produces systematic undercoverage. The same model
therefore rejects rather than releases an overconfident registration at the
$0.02$ and $0.05$\,m-per-step levels, while retaining $188/200$ success at
$0.01$\,m per step. Paired errors separate estimator effects from randomized
geometry: the known-yaw oracle remains centered near the proposed method,
whereas the EKF and endpoint registration have positive error shifts and
heavier tails. Finally, the Gazebo distributions support the aggregate
claims in Table~\ref{tab:gazebo} without assigning unwarranted precision to
any individual, non-bitwise-reproducible physics launch.

\textbf{Reproducibility.} The accompanying artifact contains the simulator,
campaign runner, tests, figure scripts, and video pipeline; a single command
regenerates every figure from the archived campaign, and every manifest
records the producing commit, build identity, and configuration.

\textbf{Long horizon and disturbance.} Over $600$\,s the task error holds
near $0.06$\,m at every tested body bias; at $5$\,cm/s, median dead
reckoning error reaches $27$\,m, matching Theorem~\ref{thm:seeking}
(oracle comparison in Remark~\ref{rem:seeking-validation}).
Fig.~\ref{fig:disturbance} shows the threshold-sharp recovery of
Proposition~\ref{prop:recovery} (Remark~\ref{rem:recovery-validation}):
large steps recalibrate at the observed recovery rates once motion
supplies views, while small at-rest steps often remain unadopted
precisely because they sit near the certified tolerance, with the task
unaffected.

\begin{figure}[t]
\centering
\includegraphics[width=\linewidth]{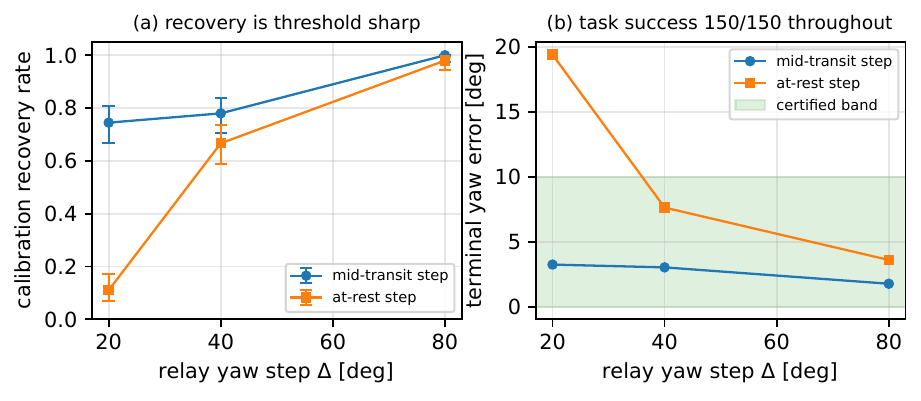}
\caption{Relay-frame disturbance study with Wilson $95\%$ intervals.
(a)~Calibration recovery rate over cases that require recovery.
(b)~Terminal yaw error against the certified band: small at-rest steps
may remain unadopted, yet task success is $150/150$ everywhere.}
\label{fig:disturbance}
\end{figure}

\begin{table}[t]
\caption{Closed-Loop Comparison}
\label{tab:baselines}
\centering
\begingroup\scriptsize\itshape
Over 200 paired randomized geometries. Station RMSE is the median with
bootstrap 95\% interval; runtime is the median estimator update cost.
Anchor-assum.\ ablation deliberately violates the globally-known-trajectory
assumption of method [8],[9] by pinning the relay-to-odometry yaw at zero
once that anchor is removed; it is not a fair competitor.
\par\endgroup
\vspace{2pt}
\footnotesize
\setlength{\tabcolsep}{3.0pt}
\begin{tabular}{@{}lcccc@{}}
\toprule
Method & Success & Station RMSE [m] & Acq.\ [s] & Update [$\mu$s]\\
\midrule
Proposed & $200/200$ & $0.064\;[0.058,0.069]$ & $14.0$ & $3.6$\\
Known-yaw oracle & $200/200$ & $0.065\;[0.059,0.072]$ & $12.8$ & $4.1$\\
Sequential EKF & $200/200$ & $0.087\;[0.080,0.095]$ & $14.2$ & $0.3$\\
Endpoint only & $194/200$ & $0.071\;[0.064,0.079]$ & $27.1$ & $0.2$\\
Fixed-lag smoother & $200/200$ & $0.064\;[0.058,0.069]$ & $14.0$ & $181.7$\\
Anchor-assum.\ abl. & $67/200$ & $35.2\;[7.4,51.6]$ & $19.2$ & $11.1$\\
No excitation & $0/200$ & \multicolumn{3}{c}{never certified}\\
\bottomrule
\end{tabular}
\end{table}

\begin{figure}[t]
\centering
\includegraphics[width=\linewidth]{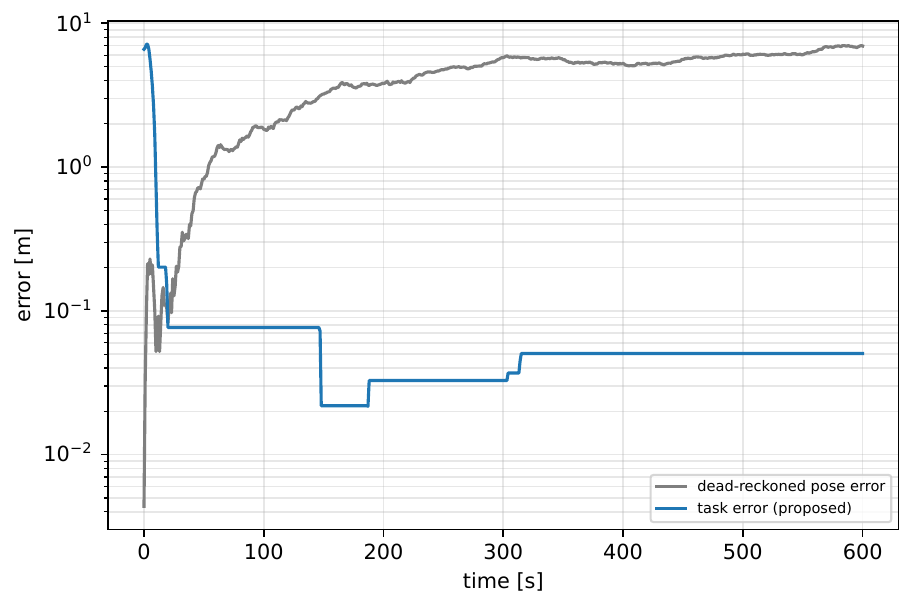}
\caption{Six-hundred-second station keeping: the task error stays within
the ultimate ball of Theorem~\ref{thm:seeking}, while dead-reckoned pose
error reaches $6.95$\,m in this representative $1$\,cm/s-bias trial.}
\label{fig:longhorizon}
\end{figure}

Table~\ref{tab:baselines} summarizes the paired comparison, and
Fig.~\ref{fig:longhorizon} shows the long-horizon contrast of
Theorem~\ref{thm:seeking}.

\section{Discussion}
\label{sec:discussion}

The fixed-lag smoother is a useful negative result rather than an omitted
success. On identical trials it matches the proposed estimator's median and
$90$th-percentile station errors ($0.064$ and $0.117$\,m) while requiring
about $50\times$ the estimator-update time. In the tested static-relay,
static-target regime, the certified closed-form registration therefore
extracts the task-relevant information already present in a window; adding a
trajectory-state optimization layer does not improve the quotient estimate.
This conclusion is deliberately scoped to the present measurement and
odometry model, not to loop-rich SLAM problems where smoothing can recover
additional state information.

The sharp refusal boundary in Fig.~\ref{fig:operating-envelope}(b) is also a
safety result. Above the supported translational-odometry regime, the
certificate does not silently turn a small residual into a precise global
claim: it withholds seeking because accumulated pose uncertainty makes the
relay-frame rotation unreliable. The diagonal and packet-only ablations show
why that behavior matters: an apparently tighter but correlation-blind
interval can substantially under-cover. Extending the deployed gate to the
full $SE(2)$ recursion of Appendix~\ref{app:certgen} is the direct route to
certifying heading uncertainty; the current heading-drift sweep remains an
empirical robustness test.

The quotient construction generalizes beyond a particular range-bearing
relay whenever an unknown but rigid local frame reports both a static task
vector and repeated ego-relative vectors, ego motion supplies metric
increments, and packets are associated in time. Differencing removes the
unknown translation, motion identifies the residual rotation, and the
recovered transform makes the task vector actionable without reconstructing
a globally anchored trajectory. Moving targets, unassociated packets,
nonrigid relay frames, or persistent frame drift require additional dynamics
or data association and are outside the present guarantees.

\section{Conclusion}
\label{sec:conclusion}

This work establishes that the loss of every global pose reference does
not preclude reliable hidden-target seeking through a single unknown-pose
range-bearing relay: the complete vehicle-relay-target configuration
remains unobservable up to a global $SE(2)$ transformation, but the
information the task requires survives on an identifiable quotient. Two
distinct vehicle views recover the relay-to-odometry yaw and the target
displacement in closed form, so motion itself is the mechanism that
converts an otherwise ambiguous sensing configuration into an actionable
one.

The quotient structure enables more than localization-free estimation. The
closed-form multi-view self-calibration estimator's uncertainty explicitly
accounts for the correlations that integrated translational odometry
induces, reducing the corresponding first-order variance to an $O(K)$
suffix sum. This certificate closes the loop between observability,
estimation, and control: the vehicle excites only until its calibration
is sufficiently informative, refuses estimates that cannot be certified,
seeks directly in task coordinates, and detects relay-frame changes
through the same uncertainty model, with zero false adoptions observed
across $200$ control trials.

The resulting ultimate task-error bound is independent of mission horizon
(Theorem~\ref{thm:seeking}), even though the global dead-reckoning error
continues to accumulate: dead reckoning grows to tens of meters while
station error stays near the centimeter scale, and across $200$
randomized geometries
the closed loop is statistically indistinguishable from an oracle handed
the true relay yaw: the paired median difference is $0.001$\,m, and its
bootstrap 95\% interval includes zero. Physics-based ROS~2/Gazebo trials
retain $30/30$ success under nominal operation, combined communication
degradation, and relay-frame disturbances.

More broadly, the result is not specific to this sensing configuration: a
vehicle need not reconstruct an unobservable global state when the control
objective lives entirely on an identifiable quotient. Future work is a
spread-growing acquisition law with an explicit threshold-crossing
guarantee, extending the certified envelope to harsher odometry, correlated
rotational drift, moving targets and relays, and the hardware validation
supported by the accompanying ROS~2 pipeline.

\section*{Acknowledgment}
OpenAI Codex and Anthropic Claude assisted with implementation and
verification tooling for the simulations and ROS~2/Gazebo experiments in
Sec.~\ref{sec:experiments}, and with language editing throughout the
manuscript. The authors independently verified every mathematical claim,
code path, figure, and reported result and bear full responsibility for the
content.


\appendices

\section{Quotient Geometry and Local Rank}
\label{app:quotient}

Fix a window of $K$ views with known odometric poses $s_1,\dots,s_K$
(Sec.~\ref{sec:problem}). The unknowns are the quotient
$(\theta,x_o,p_o)\in S^1\times\R^2\times\R^2$, a five-dimensional space,
and the relay reports, at each view $k$, the pair
$(\ell_k^v,\ell_k^t)$ given by \eqref{eq:reduced}. Let
\begin{equation}
S = \begin{pmatrix} 0 & -1 \\ 1 & 0\end{pmatrix},
\qquad
\frac{d}{d\theta}\Rot(-\theta)v = -S\Rot(-\theta)v,
\label{eq:Sgen}
\end{equation}
be the generator of planar rotation. Differentiating
$\ell_k^v=\Rot(-\theta)(s_k-x_o)$ and $\ell_k^t=\Rot(-\theta)(p_o-x_o)$
termwise gives the stacked Jacobian of the noiseless packet map at view
$k$, with respect to $(\theta,x_o,p_o)$:
\begin{equation}
J_k = \begin{pmatrix} -S\ell_k^v & -\Rot(-\theta) & 0 \\
-S\ell_k^t & -\Rot(-\theta) & \Rot(-\theta) \end{pmatrix} \in \R^{4\times5}.
\label{eq:stackedjac}
\end{equation}

\textbf{Target rows carry no view index.} Since $p_o,x_o,\theta$ do not
depend on $k$, the target-packet block
$(-S\ell_k^t,-\Rot(-\theta),\Rot(-\theta))$ is \emph{identical for every
$k$}: repeating a view yields further copies of the same two rows, which
span no new direction. Repeated target measurements therefore strictly
improve precision, since each carries an independent noise realization
that reduces variance, but they contribute no additional rank to the
quotient beyond what one such row already supplies. This is the fact the
main text
uses implicitly when eliminating $p_o$ from the registration estimator
\eqref{eq:registration}, which depends on the vehicle packets $\ell_k^v$
alone.

\textbf{Single-view rank.} At $K=1$, stacking $J_1$ gives a $4\times5$
matrix. Its two $\Rot(-\theta)$ blocks are full rank $2$ on the disjoint
column pairs $(x_o,p_o)$, so the only possible rank deficiency is between
the $\theta$-column and those blocks. Direct substitution shows $J_1$
annihilates the tangent, at $\alpha=0$, of the one-parameter family
$(\theta+\alpha,\, s_1-\Rot(\theta+\alpha)\ell_1^v,\,\cdot)$ of
Corollary~\ref{cor:repeated} (that family reproduces $\ell_1^v,\ell_1^t$ to
first order for every $\alpha$), and no other direction:
$\operatorname{rank} J_1 = 4$. This exhibits the ``rank four of five'' of
Corollary~\ref{cor:repeated} as the explicit null space of
\eqref{eq:stackedjac}, rather than only asserting it.

\textbf{Two-view rank.} A second view with $s_2\ne s_1$ contributes a
further vehicle-packet block $(-S\ell_2^v,-\Rot(-\theta),0)$. Because the
$x_o$ block $-\Rot(-\theta)$ is view-independent and the target columns
are zero in vehicle rows, the row space of the stacked Jacobian contains
the \emph{difference} of the two vehicle-packet blocks,
\begin{equation*}
\Delta J_v \;=\; \bigl[\,-S\bigl(\ell_2^v-\ell_1^v\bigr),\;\; 0,\;\; 0\,\bigr].
\end{equation*}
The reduced model gives $s_2-s_1=\Rot(\theta)\bigl(\ell_2^v-\ell_1^v\bigr)$,
and $\Rot(\theta)$ is invertible, so $s_2\ne s_1$ forces
$\ell_2^v\ne\ell_1^v$. Note the two relay vectors need only be
\emph{different vectors}, not differently directed: two vehicle positions
on the same ray from the relay give parallel but unequal $\ell^v$, and the
argument covers that case (a parallelism-based argument would not). The
single-view null direction identified above has a nonzero
$\theta$-component; applying $\Delta J_v$ to it yields
$-S(\ell_2^v-\ell_1^v)\ne0$ since $S$ is invertible. The second view thus
supplies the missing yaw direction, no nonzero tangent is annihilated, and
the stacked Jacobian over the two views has full rank $5$: this is the
local (first-order, differential) counterpart of the global closed-form
identifiability already proved constructively by
Theorem~\ref{thm:twoview}.

\section{Native Range-Bearing Information}
\label{app:fisher}

The main text's certificate (Lemma~\ref{lem:conditional},
Theorem~\ref{thm:certificate}) is stated for isotropic Cartesian packet
noise of variance $\sigma^2$, with $\sigma_{\perp,k}^2$ denoting its
component perpendicular to $b_k$. Here we derive $\sigma_{\perp,k}^2$ from
the sensor's native, anisotropic range-bearing noise model
$(\sigma_r,\sigma_\beta)$ of \eqref{eq:packets}, and derive the yaw Fisher
information this induces.

For a relative vector $u\in\R^2\setminus\{0\}$ with polar representation
$(r,\beta)=(\lVert u\rVert,\angle u)$, the Jacobian mapping a Cartesian
perturbation $du$ to the induced polar perturbation $(dr,d\beta)$ is
\begin{equation}
P(u) = \begin{pmatrix} u^{\!\top}/\lVert u\rVert \\
(Su)^{\!\top}/\lVert u\rVert^2 \end{pmatrix} \in \R^{2\times2},
\label{eq:polarjac}
\end{equation}
with $S$ as in \eqref{eq:Sgen}: the first row is the radial gradient of
$r$, the second the tangential gradient of $\beta$. Range and bearing are
measured with independent native noise $(\sigma_r,\sigma_\beta)$, so by
the standard Fisher-information transformation rule under a bijective
reparametrization, the information matrix about $u$ induced by one
range-bearing measurement is, to first order,
\begin{equation}
\begin{split}
Q(u) &= P(u)^{\!\top}\operatorname{diag}(\sigma_r^{-2},\sigma_\beta^{-2})\,P(u)\\
&= \frac{uu^{\!\top}}{\sigma_r^2\lVert u\rVert^2}
+ \frac{(Su)(Su)^{\!\top}}{\sigma_\beta^2\lVert u\rVert^4}.
\end{split}
\label{eq:qform}
\end{equation}
The two terms are orthogonal rank-one projectors (onto $u$ and onto
$Su\perp u$) scaled by $1/\sigma_r^2$ and $1/(\sigma_\beta^2\lVert
u\rVert^2)$: range noise is informative only along $u$ (radial), bearing
noise only along $Su$ (tangential, perpendicular to $u$). The
$\lVert u\rVert^4$ denominator (equivalently $\lVert u\rVert^2$ after
normalizing $Su$ to a unit projector) is geometric, not cosmetic: a
bearing error $\sigma_\beta$ corresponds to a lateral \emph{position}
error $\lVert u\rVert\sigma_\beta$ that grows with range, so the position
information it carries decays as $1/\lVert u\rVert^2$. Dimensional
consistency confirms the power: both terms of $Q(u)$ must carry units of
$1/\mathrm{length}^2$.

\textbf{Reduction to $\sigma_{\perp,k}^2$.} The registration estimator
\eqref{eq:registration} is a ratio of cross and dot products of $b_k$
against $a_k$; its sensitivity to a packet perturbation $\nu_k$ of
$\ell_k^v$ (Step 2 of Theorem~\ref{thm:certificate}'s proof) is
$d\hat\theta=(1/W)\sum_k w_k(\nu_k\times b_k)$, which depends on $\nu_k$
\emph{only through its component perpendicular to the centered vector
$b_k$}. This perpendicular direction is defined by
$b_k=\ell_k^v-\bar\ell_w$, whereas native range noise acts radially along
the \emph{raw} measurement vector $\ell_k^v$; the two directions differ in
general, so range noise does \emph{not} drop out of the yaw variance.
Writing the native per-packet Cartesian covariance as
\begin{equation}
\Sigma_{\ell,k} \;=\; \sigma_r^2\,\hat u_{r,k}\hat u_{r,k}^{\!\top}
\;+\; r_k^2\sigma_\beta^2\,\hat u_{t,k}\hat u_{t,k}^{\!\top},
\qquad \hat u_{t,k}=S\hat u_{r,k},
\label{eq:pktcov}
\end{equation}
with $\hat u_{r,k}=\ell_k^v/\lVert\ell_k^v\rVert$ and
$r_k=\lVert\ell_k^v\rVert$, the packet-noise term of the certificate is
\begin{equation}
\operatorname{Var}_{\rm pkt}(\hat\theta)
= \frac{1}{W^2}\sum_k w_k^2\,(Sb_k)^{\!\top}\Sigma_{\ell,k}\,(Sb_k),
\label{eq:pktvar}
\end{equation}
equivalently, in the main text's form
$\sum_k w_k^2\lVert b_k\rVert^2\sigma_{\perp,k}^2/W^2$,
\begin{equation}
\sigma_{\perp,k}^2
= \hat b_{\perp,k}^{\!\top}\,\Sigma_{\ell,k}\,\hat b_{\perp,k}
= \sigma_r^2\bigl(\hat u_{r,k}\!\cdot\!\hat b_{\perp,k}\bigr)^2
+ r_k^2\sigma_\beta^2\bigl(\hat u_{t,k}\!\cdot\!\hat b_{\perp,k}\bigr)^2,
\label{eq:sigmaperp}
\end{equation}
with $\hat b_{\perp,k}=Sb_k/\lVert b_k\rVert$. Both native channels
contribute whenever the window geometry makes them relevant; the radial
channel vanishes only in the special geometry $b_k\parallel\ell_k^v$ for
every $k$. This projection is exactly what the deployed estimator
evaluates per packet.

\begin{remark}[Range noise can enter at first order]
Take $\ell_1^v=(1,0)^{\!\top}$, $\ell_2^v=(0,1)^{\!\top}$ with equal
weights, so $b_1=(\tfrac12,-\tfrac12)^{\!\top}$. A radial perturbation at
view 1 is parallel to $(1,0)^{\!\top}$, and its cross product with $b_1$
is $\tfrac12\ne0$: range noise contributes to first-order yaw uncertainty.
Only the component of $\nu_k$ along $b_k$ cancels identically, not the
radial direction $\hat u_{r,k}$.
\end{remark}

\textbf{Fisher matrix and yaw Schur complement.} Using the Jacobian block
$H_k=[-S\ell_k^v,\,-\Rot(-\theta)]$ of \eqref{eq:stackedjac} restricted to
$(\theta,x_o)$, the Fisher information about $(\theta,x_o)$ contributed by
view $k$ is $I_k = H_k^{\!\top} Q(\ell_k^v)\, H_k$, a $3\times3$ matrix;
$Q$ already encodes the measurement noise, and estimator weights do not
enter the information bound. Summing over the window gives
$I(\theta,x_o)=\sum_k I_k$, and the information about $\theta$ alone, i.e.
the information remaining after optimally marginalizing the nuisance
$x_o$, is the exact Schur complement
\begin{equation}
I_\theta \;=\; I_{\theta\theta}
\;-\; I_{\theta x}\,I_{xx}^{-1}\,I_{x\theta}.
\label{eq:ithetagen}
\end{equation}
This is the form the deployed conditional-information diagnostic
evaluates per window, with the per-packet anisotropic $Q(\ell_k^v)$ of
\eqref{eq:qform}. A closed-form centering identity is available exactly
when each view's noise is isotropic with scalar information $q_k$, i.e.
$Q(\ell_k^v)=q_k I_2$: eliminating $x_o$ then reduces to centering the
regressors at their information-weighted mean, giving the
information-weighted spread
\begin{equation}
I_\theta = S_w := \sum_k q_k\,\bigl\lVert \ell_k^v-\bar\ell_q\bigr\rVert^2,
\qquad
\bar\ell_q=\frac{\sum_k q_k\,\ell_k^v}{\sum_k q_k}.
\label{eq:sw}
\end{equation}
In the homogeneous case $q_k=\sigma^{-2}$ this is $S_v/\sigma^2$ with the
unweighted spread $S_v$ of Lemma~\ref{lem:conditional}, the special case
quoted in the main text. For genuinely anisotropic $Q(\ell_k^v)$, e.g.
native range-bearing noise with $\sigma_r \ne r_k\sigma_\beta$, no single
centering point reproduces the matrix-weighted marginalization, and
\eqref{eq:ithetagen} must be evaluated as written.

\section{Joint Quotient Covariance under General $SE(2)$ Odometry}
\label{app:certgen}

The main-text certificate retains the full cross-view correlation generated
by integrated translational odometry while remaining an $O(K)$ scalar
calculation. This section gives the corresponding first-order construction
for body-frame $SE(2)$ increments and then propagates the result to the
complete control statistic $\xi_j=(\theta,e_j)\in S^1\times\mathbb R^2$.
Every covariance below belongs to one independent increment; no increment
covariance is differenced from its neighbor.

\subsection{Translation and heading increments}

Let $z_k=(s_k,\phi_k)$ be the odometric position and heading at view $k$,
and let $u_k=(\Delta r_k,\Delta\phi_k)$ be the nominal body-frame increment
from $k-1$ to $k$, with the discrete reconstruction convention
$s_k=s_{k-1}+R(\phi_{k-1})\Delta r_k$ and
$\phi_k=\phi_{k-1}+\Delta\phi_k$. For an independent, zero-mean increment perturbation
$\varepsilon_k=(\varepsilon_k^r,\varepsilon_k^\phi)$ with covariance
$Q_k\succeq0$, the first-order dead-reckoning error obeys
\begin{equation}
 \begin{aligned}
 \delta z_k&=F_k\delta z_{k-1}+L_k\varepsilon_k,\\
 F_k&=\begin{bmatrix}
 I_2&R(\phi_{k-1})S\Delta r_k\\0_{1\times2}&1
 \end{bmatrix},\\
 L_k&=\begin{bmatrix}R(\phi_{k-1})&0\\0_{1\times2}&1\end{bmatrix}.
 \end{aligned}
 \label{eq:se2error}
\end{equation}
where $S=\left[\begin{smallmatrix}0&-1\\1&0\end{smallmatrix}\right]$.
The first view fixes the odometry-frame gauge, so $\delta z_1=0$. With
$\Phi_{k,m}=F_kF_{k-1}\cdots F_{m+1}$ and $\Phi_{m,m}=I_3$,
\begin{equation}
 \operatorname{Cov}(\delta z_i,\delta z_j)
 =\sum_{m=2}^{\min(i,j)}
 \Phi_{i,m}L_mQ_mL_m^\top\Phi_{j,m}^\top .
 \label{eq:se2crosscov}
\end{equation}
This form retains anisotropy, translation--heading correlation within an
increment, and the long-range position correlation generated by integrated
heading error.

Let $E_p=[I_2\;0]$ select position and set
$h_k=E_p^\top g_k\in\mathbb R^3$, where $g_k$ is the registration
sensitivity in Theorem~\ref{thm:certificate}. Holding the registration
weights fixed at their noiseless linearization values, define
\begin{equation}
 r_K=h_K,\qquad r_m=h_m+F_{m+1}^\top r_{m+1},\qquad
 H_m=L_m^\top r_m .
 \label{eq:se2backward}
\end{equation}
Reordering the independent increment contributions in
$\delta\hat\theta_{\rm odo}=\sum_k h_k^\top\delta z_k$ gives
\begin{equation}
 \operatorname{Var}_{\rm odo}(\hat\theta)
 =\sum_{m=2}^{K}H_m^\top Q_m H_m .
 \label{eq:varodose2}
\end{equation}
The backward recursion is $O(K)$ despite the dense covariance in
\eqref{eq:se2crosscov}. If heading error is absent, the position dynamics
reduce to partial sums, $L_kQ_kL_k^\top=\Delta v_kI_2$, and
$r_m=\sum_{k\ge m}g_k=G_m$; therefore \eqref{eq:varodose2} reduces exactly
to the translation-only suffix sum in \eqref{eq:varfull}.

\begin{remark}[Deployed and generalized certificates]
The online implementation and every coverage statement in the main text use
the translation-only specialization \eqref{eq:varfull}. Equation
\eqref{eq:varodose2} is the analytic extension needed to certify general
$SE(2)$ odometry, but it does not retroactively make the existing
heading-drift sweep a coverage experiment. Empirically, all $200/200$ trials
still complete at $0.5^\circ$/s heading drift and median station RMSE rises
from $0.064$ to $0.080$\,m; this is a robustness result, not a coverage claim
for the deployed certificate.
\end{remark}

\subsection{Complete yaw--task covariance}

The controller acts on $e_j=R(\theta)d_j$, where
$d_j=\ell_j^t-\ell_j^v$. Let $C_j^v$ and $C_j^t$ denote the local Cartesian
covariances obtained from the native range-bearing packets and write their
independent perturbations as $\nu_j^v$ and $\nu_j^t$. Then
$\delta d_j=\nu_j^t-\nu_j^v$ has covariance
$C_j^d=C_j^t+C_j^v$. The same vehicle packet also contributes to the
registration angle. With
\begin{equation*}
 a_j^v=\frac{w_j}{W}(-Sb_j),
 \qquad
 \delta\hat\theta_{\rm pkt}=\sum_k(a_k^v)^\top\nu_k^v,
\end{equation*}
the cross-covariance discarded by a diagonal approximation is
\begin{equation}
 c_j:=\operatorname{Cov}(\delta d_j,\delta\hat\theta)
 =-C_j^v a_j^v .
 \label{eq:taskcross}
\end{equation}
Odometry increments are independent of relay packets and hence contribute
to $\sigma_\theta^2=\operatorname{Var}(\hat\theta)$ through
\eqref{eq:varodose2}, but add no term to \eqref{eq:taskcross}.

For small yaw error, set $q_j=R(\theta)Sd_j$ and linearize
$\delta e_j=q_j\delta\hat\theta+R(\theta)\delta d_j$. The complete
first-order covariance of $\xi_j=(\theta,e_j)$ is
\begin{equation}
 \Sigma_{\xi_j}=A_j
 \begin{bmatrix}\sigma_\theta^2&c_j^\top\\c_j&C_j^d\end{bmatrix}A_j^\top,
 \qquad
 A_j=\begin{bmatrix}1&0_{1\times2}\\q_j&R(\theta)\end{bmatrix}.
 \label{eq:quotientcov}
\end{equation}
Writing its lower-right block as $\Sigma_{e_j}$, the first-order Gaussian
confidence ellipse
\begin{equation}
 \mathcal E_{e,j}(1-\alpha)=
 \{\delta e:\delta e^\top\Sigma_{e_j}^{-1}\delta e
 \le\chi^2_{2,1-\alpha}\}
 \label{eq:taskellipse}
\end{equation}
has probability $1-\alpha$. A scalar supervisor gate follows without
discarding the ellipse's worst direction:
\begin{equation}
 \Pr\!\left\{\lVert\delta e_j\rVert\le
 \sqrt{\chi^2_{2,1-\alpha}\lambda_{\max}(\Sigma_{e_j})}\right\}
 \ge 1-\alpha .
 \label{eq:taskchance}
\end{equation}
The three-dimensional ellipsoid formed from $\Sigma_{\xi_j}$ similarly
provides a joint yaw--task certificate using $\chi^2_{3,1-\alpha}$. Thus the
deployed yaw gate is a scalar projection of a complete quotient-level chance
certificate rather than a separate uncertainty construction.

\section{Hybrid Control and Change Detection}
\label{app:control}

This appendix collects the expanded proofs behind the closed-loop
guarantees of Sec.~\ref{sec:closedloop}: the realization gap between the
virtual seeking law and the deployed unicycle controller
(Lemma~\ref{lem:unicycle}, stated and proved in Sec.~\ref{sec:closedloop}
itself once the certified-seeking theorem is in hand), the acquisition
dichotomy, the change-detection guarantee together with its exact
detection-probability formula, and the small-step no-adoption case.

\subsection{Expanded Proof of Proposition~\ref{prop:acquisition}}
\label{app:acquisition}

\emph{Step 1 (spread growth on the arc).}
Let the acquisition arc have speed $v_e$ and curvature $\kappa$, sampled at
packet period $T_s$, and let the window hold the most recent $K$ views. Two
views separated by arc time $\tau$ have chord length
\begin{equation*}
c(\tau) = \frac{2}{\kappa}\Bigl\lvert\sin\Bigl(\frac{\kappa v_e \tau}{2}\Bigr)\Bigr\rvert
\;\ge\; \frac{2 v_e \tau}{\pi}, \qquad \kappa v_e\tau \le \pi.
\end{equation*}
The relay map \eqref{eq:reduced} is an isometry of the view set, so the
centered relay vectors $b_k$ inherit the same pairwise distances as the
poses, and the window spread therefore satisfies
\begin{equation*}
S_v \;\ge\; \frac{1}{K}\sum_{k} c\bigl(\lvert k - \bar k\rvert T_s\bigr)^2,
\end{equation*}
a strictly positive, geometry-determined quantity that is nondecreasing as
the window fills.

\emph{Step 2 (packet term decay).}
Suppose $0<\underline w\le w_k\le\bar w$ and
$\sigma_{\perp,k}^2\le\bar\sigma^2$. Since
$\sum_k w_k^2\lVert b_k\rVert^2\le\bar w W$ and
$W\ge\underline w S_v$, the packet contribution to
\eqref{eq:varfull} satisfies the correct weighted bound
\begin{equation*}
 \frac{\sum_k w_k^2\lVert b_k\rVert^2\sigma_{\perp,k}^2}{W^2}
 \le \frac{\bar\sigma^2\bar w}{\underline w S_v}.
\end{equation*}
With $S_v$ bounded below by Step 1 and growing until the window saturates,
this upper bound decreases to a positive constant determined by the
saturated-arc spread $S_v^\infty$ and the weight ratio
$\bar w/\underline w$.

\emph{Step 3 (odometry floor).}
Each inter-view interval contributes increment variance
$\Delta v = \sigma_s^2 T_s/\delta t$ per axis, so the exact odometry term
of a saturated deterministic arc is
\begin{equation*}
 v_{\rm odo}^{\infty}:=\Delta v
 \sum_{m=2}^{K}\lVert G_m^{\infty}\rVert^2.
\end{equation*}
The sensitivities scale as
$\lVert g_k\rVert\sim\lVert b_k\rVert/S_v$, which gives the separate
order relation
\begin{equation*}
 v_{\rm odo}^{\infty}
 \asymp \frac{\sigma_s^2 T_s K}{\delta t\,S_v^{\infty}}
\end{equation*}
up to fixed arc-geometry constants. The first expression is the exact
first-order floor for the saturated reference arc; the second only explains
its scaling and is not used as an equality or lower bound.

\emph{Step 4 (threshold crossing, not monotonicity).}
By Steps 1--3, the packet-noise upper bound decreases while the window fills,
and the deterministic reference geometry reaches the exact odometry term
$v_{\rm odo}^{\infty}$ at saturation. This does not make
$B(\mathcal K)$ monotone once the window slides: successive windows share
only $K-1$ views, so noisy realized weights and geometry can move the
certificate in either direction. Two conclusions remain valid.
\emph{Saturated-window refusal is exact for the reference arc}: if
$v_{\rm odo}^{\infty}\ge(\bar\delta/z_{0.975})^2$, its nonnegative packet
term makes every saturated reference window fail the confidence gate.
\emph{Certification remains a realized threshold-crossing condition}: if
the exact saturated total variance lies below threshold with strict margin,
the reference fill trajectory reaches a certifiable window in finite time.
For noisy sliding windows, obtaining $L$ consecutive certified windows
requires an additional stationarity or ergodicity condition; the main
proposition therefore states the guarantee conditionally on those realized
windows and leaves their frequency to the empirical validation in
Sec.~\ref{sec:experiments}.
\hfill$\blacksquare$

\subsection{Expanded Proof of Proposition~\ref{prop:recovery}}
\label{app:recovery}

\emph{Step 1 (false-alarm bound).}
Under no step, a certified window estimate satisfies
$\hat\theta_w = \theta + \epsilon_w$ with
$\epsilon_w/\sqrt{\operatorname{var}(\hat\theta_w)}$ approximately
standard normal by the coverage calibration of
Remark~\ref{rem:coverage}. The standardized discrepancy against the
control path is then $\chi^2_1$ to first order, and a single window
exceeds the threshold $\chi^2_{1,1-\alpha_\chi}$ with probability
$\alpha_\chi$. Windows separated by a full turnover share no packets, so
their exceedances are independent up to the common control path, which is
held fixed between adoptions; the probability that $L$ consecutive
certified windows spanning $c$ turnovers all exceed the threshold is at
most $\alpha_\chi^{\,c}$. Rate limiting can only decrease the adoption
probability further.

\emph{Step 2 (detection probability).}
After a step of size $\Delta$, post-step certified windows estimate
$\theta + \Delta$ while the control path holds $\theta$. The standardized
discrepancy is noncentral $\chi^2_1$ with noncentrality
$\lambda = \Delta^2/(\operatorname{var}(\hat\theta_w) +
\operatorname{var}_{\rm ctrl})$, so a single post-step window's detection
probability is exactly
\begin{equation}
P_D = \Pr\bigl(\chi^2_1(\lambda) > \tau\bigr) = 1 - F_{\chi^2_1(\lambda)}(\tau),
\qquad \tau=\chi^2_{1,1-\alpha_\chi},
\label{eq:pdetect}
\end{equation}
with $F_{\chi^2_1(\lambda)}$ the noncentral-$\chi^2$ CDF with one degree of
freedom. Since $F_{\chi^2_1(\lambda)}(\tau)\to0$ as $\lambda\to\infty$ for
any fixed $\tau$, $P_D\to1$ as $\Delta$ grows relative to the certified
band, recovering the qualitative claim of the main text quantitatively.
Once every window votes ($P_D\approx1$), the persistence counter advances
by one per packet, so adoption completes within $L$ packets after the
first fully post-step window, that is, within $L$ windows plus one
turnover of motion-supplied certified views.

\emph{Step 3 (small steps need no action).}
Whenever the combined error satisfies $|\tilde\theta_{\rm old}|+|\Delta|
\le\delta_c<\pi/2$ (Proposition~\ref{prop:recovery}(iii)), the triangle
inequality bounds the post-step yaw error by exactly this combined
quantity, so the perturbed yaw error still satisfies the hypothesis of
Theorem~\ref{thm:seeking} with $\delta_c$ in place of $\bar\delta$, and
contraction to the ultimate ball $\bar\eta/\cos\delta_c$ continues without
any adoption. This is the threshold-sharp behavior observed in the
validation campaign.
\hfill$\blacksquare$

\section{Expanded Evidence}
\label{app:tables}

All values below are from the same provenance-locked campaigns cited in
Sec.~\ref{sec:experiments} and the Gazebo campaign of
Table~\ref{tab:gazebo}; medians are reported with 90th percentiles or
bootstrap/Wilson 95\% intervals as indicated. No new campaign was run to
produce this appendix.

\subsection{Certificate Conditioning}

Table~\ref{tab:conditioning} is the full $20$-cell chord-length /
view-count conditioning study underlying Remark~\ref{rem:coverage}'s
pooled statistics: $1{,}000$ realizations per cell. It shows precisely
where the first-order certificate is weakest, namely short chords with
only two views, and where it is essentially exact.

\begin{table}[h]
\caption{Certificate conditioning: empirical-to-predicted yaw-variance
ratio and realized coverage vs. chord length and view count
($1{,}000$ trials/cell).}
\label{tab:conditioning}
\centering
\footnotesize
\setlength{\tabcolsep}{3pt}
\begin{tabular}{@{}lccccc@{}}
\toprule
Chord [m] & Views & Ratio & Cov90 & Cov95 & Cov99\\
\midrule
0.25 & 2  & $1.742$ & $0.869$ & $0.916$ & $0.971$\\
0.25 & 4  & $0.990$ & $0.911$ & $0.949$ & $0.978$\\
0.25 & 8  & $0.615$ & $0.963$ & $0.983$ & $0.997$\\
0.25 & 16 & $0.560$ & $0.978$ & $0.994$ & $0.998$\\
0.5  & 2  & $1.152$ & $0.910$ & $0.960$ & $0.988$\\
0.5  & 4  & $1.104$ & $0.893$ & $0.945$ & $0.985$\\
0.5  & 8  & $0.869$ & $0.923$ & $0.962$ & $0.994$\\
0.5  & 16 & $0.824$ & $0.938$ & $0.975$ & $0.997$\\
1.0  & 2  & $0.995$ & $0.915$ & $0.962$ & $0.993$\\
1.0  & 4  & $1.088$ & $0.888$ & $0.946$ & $0.983$\\
1.0  & 8  & $0.943$ & $0.909$ & $0.951$ & $0.993$\\
1.0  & 16 & $0.927$ & $0.916$ & $0.969$ & $0.996$\\
2.0  & 2  & $0.949$ & $0.907$ & $0.960$ & $0.997$\\
2.0  & 4  & $1.061$ & $0.892$ & $0.944$ & $0.987$\\
2.0  & 8  & $0.947$ & $0.904$ & $0.957$ & $0.993$\\
2.0  & 16 & $0.946$ & $0.915$ & $0.964$ & $0.994$\\
4.0  & 2  & $0.919$ & $0.906$ & $0.958$ & $0.999$\\
4.0  & 4  & $0.987$ & $0.900$ & $0.960$ & $0.992$\\
4.0  & 8  & $0.987$ & $0.906$ & $0.947$ & $0.992$\\
4.0  & 16 & $1.007$ & $0.888$ & $0.949$ & $0.992$\\
\bottomrule
\end{tabular}
\end{table}

The hardest cell ($0.25$\,m chord, two views) overstates precision by
$74\%$ (ratio $1.74$) and undercovers the nominal $95\%$ level at
$91.6\%$: with only two barely-separated views, the first-order
linearization is least accurate. Every cell at $\ge4$ views, or $\ge0.5$\,m
chord, recovers coverage within a few points of nominal, and by $1$\,m
chord the ratio is within $10\%$ of unity at every view count tested. This
conditioning map, not a generic noise sweep, is what tells the certificate
apart from an unconditional confidence claim: the acquisition arc
(Sec.~\ref{sec:closedloop}) is deliberately built to spend time at the
well-conditioned end of this table before certifying.

\subsection{Correlated-Odometry Certificate: Full Coverage Table}

Table~\ref{tab:oducoverage} is the complete $10$-condition table
underlying the pooled statistics of Remark~\ref{rem:coverage}: $K\in\{8,16\}$
views, translational noise $0.25$--$5$\,cm per increment, $2{,}000$
realizations per cell.

\begin{table}[h]
\caption{Correlated-odometry certificate coverage vs. increment noise and
window size ($2{,}000$ trials/cell).}
\label{tab:oducoverage}
\centering
\footnotesize
\setlength{\tabcolsep}{4pt}
\begin{tabular}{@{}ccccc@{}}
\toprule
$\sigma_s$ [cm/incr.] & Views & Ratio & Cov95 & Cov99\\
\midrule
$0.25$ & 8  & $0.969$ & $0.9485$ & $0.988$\\
$0.25$ & 16 & $1.047$ & $0.9425$ & $0.9885$\\
$0.5$  & 8  & $1.066$ & $0.9395$ & $0.9905$\\
$0.5$  & 16 & $1.011$ & $0.951$  & $0.989$\\
$1$    & 8  & $0.959$ & $0.950$  & $0.9885$\\
$1$    & 16 & $1.022$ & $0.943$  & $0.9915$\\
$2$    & 8  & $0.942$ & $0.9555$ & $0.9915$\\
$2$    & 16 & $0.979$ & $0.9495$ & $0.9895$\\
$5$    & 8  & $0.915$ & $0.9525$ & $0.9885$\\
$5$    & 16 & $0.880$ & $0.946$  & $0.9885$\\
\bottomrule
\end{tabular}
\end{table}

Coverage stays within about two points of the nominal $95\%$ level across
the full range from $0.25$ to $5$\,cm per increment and both tested window
sizes, with no monotone trend toward undercoverage at the harshest noise
level tested. This is consistent with the certificate remaining a
first-order (not asymptotic) approximation whose quality is set by the
conditioning of
Table~\ref{tab:conditioning} more than by the absolute odometry noise
level.

\subsection{Odometry Refusal Boundary}

Table~\ref{tab:odom} (main text, reproduced from the same campaign)
already distinguishes successful certification from principled refusal;
no wider odometry range was tested, since $2$--$5$\,cm per increment
already brackets the transition and both refusal rows show identical
behavior (never released).

\begin{table}[h]
\caption{Odometry degradation (proposed method). Refusal rows never
certify and never release seeking; the residual distance is the initial
geometry, reported for completeness, not as tracking error.}
\label{tab:odom}
\centering
\footnotesize
\setlength{\tabcolsep}{3pt}
\begin{tabular}{@{}lcccc@{}}
\toprule
Axis & Level & Outcome & Success & Median RMSE [m]\\
\midrule
$\sigma_s$ [cm/incr.] & $0.5$ & certified & $200/200$ & $0.063$\\
 & $1$ & certified & $188/200$ & $0.066$\\
 & $2$ & refusal & $0/200$ & (never released)\\
 & $5$ & refusal & $0/200$ & (never released)\\
bias [cm/s] & $2$ & certified & $200/200$ & $0.065$\\
 & $5$ & certified & $200/200$ & $0.071$\\
scale & $10\%$ & certified & $200/200$ & $0.066$\\
heading drift [deg/s] & $0.5$ & certified & $200/200$ & $0.080$\\
\bottomrule
\end{tabular}
\end{table}

\subsection{Relay Sensing and Communication Stress, Full Sweeps}

Tables~\ref{tab:relaynoise} and~\ref{tab:comm} are the complete
one-factor-at-a-time sweeps summarized in Fig.~\ref{fig:robustness}: $200$
randomized trials per cell, all $200/200$ successful.

\begin{table}[h]
\caption{Relay sensing noise sweep (proposed method).}
\label{tab:relaynoise}
\centering
\footnotesize
\setlength{\tabcolsep}{4pt}
\begin{tabular}{@{}lcccc@{}}
\toprule
Axis & Level & Success & Median RMSE [m] & q90 [m]\\
\midrule
range [m] & $0.025$ & $200/200$ & $0.067$ & $0.176$\\
 & $0.05$ & $200/200$ & $0.066$ & $0.156$\\
 & $0.10$ & $200/200$ & $0.060$ & $0.111$\\
 & $0.20$ & $200/200$ & $0.055$ & $0.086$\\
 & $0.40$ & $200/200$ & $0.078$ & $0.116$\\
bearing [deg] & $0.25$ & $200/200$ & $0.094$ & $0.178$\\
 & $0.5$ & $200/200$ & $0.088$ & $0.162$\\
 & $1$ & $200/200$ & $0.060$ & $0.114$\\
 & $2$ & $200/200$ & $0.087$ & $0.151$\\
 & $4$ & $200/200$ & $0.170$ & $0.288$\\
\bottomrule
\end{tabular}
\end{table}

\begin{table}[h]
\caption{Communication stress (proposed method).}
\label{tab:comm}
\centering
\footnotesize
\setlength{\tabcolsep}{4pt}
\begin{tabular}{@{}lcccc@{}}
\toprule
Axis & Level & Success & Median RMSE [m] & q90 [m]\\
\midrule
dropout & $10\%$ & $200/200$ & $0.056$ & $0.103$\\
 & $30\%$ & $200/200$ & $0.069$ & $0.105$\\
 & $50\%$ & $200/200$ & $0.080$ & $0.125$\\
delay [s] & $0.10$ & $200/200$ & $0.062$ & $0.103$\\
 & $0.25$ & $200/200$ & $0.062$ & $0.101$\\
 & $0.50$ & $200/200$ & $0.057$ & $0.099$\\
jitter [s] & $0.05$ & $200/200$ & $0.056$ & $0.106$\\
 & $0.10$ & $200/200$ & $0.063$ & $0.101$\\
 & $0.20$ & $200/200$ & $0.068$ & $0.109$\\
outliers & $1\%$ & $200/200$ & $0.085$ & $0.128$\\
 & $5\%$ & $200/200$ & $0.117$ & $0.187$\\
 & $10\%$ & $200/200$ & $0.149$ & $0.242$\\
 & $20\%$ & $200/200$ & $0.204$ & $0.321$\\
\bottomrule
\end{tabular}
\end{table}

\subsection{Paired Baselines and Runtime Distributions}

Table~\ref{tab:baselinesfull} expands Table~\ref{tab:baselines} with $90$th
percentiles and Wilson success intervals over the same $200$ paired
randomized geometries.

\begin{table}[h]
\caption{Baseline comparison, full distribution (200 paired trials; Wilson
95\% success interval, station RMSE median and q90).}
\label{tab:baselinesfull}
\centering
\footnotesize
\setlength{\tabcolsep}{2.5pt}
\resizebox{\columnwidth}{!}{%
\begin{tabular}{@{}lcccc@{}}
\toprule
Method & Success (Wilson 95\%) & Med.\ RMSE [m] & q90 [m] & Update [$\mu$s]\\
\midrule
Proposed & $200/200$ $[0.981,1]$ & $0.0639$ & $0.117$ & $3.60$\\
Prop.+smoother & $200/200$ $[0.981,1]$ & $0.0639$ & $0.117$ & $181.7$\\
Known-yaw oracle & $200/200$ $[0.981,1]$ & $0.0653$ & $0.110$ & $4.06$\\
Sequential EKF & $200/200$ $[0.981,1]$ & $0.0868$ & $0.172$ & $0.28$\\
Endpoint-only & $194/200$ $[0.936,0.986]$ & $0.0708$ & $0.169$ & $0.25$\\
Fixed-decay exc. & $200/200$ $[0.981,1]$ & $0.0625$ & $0.120$ & $3.49$\\
No excitation & $0/200$ $[0,0.019]$ & \multicolumn{2}{c}{never certified} & $3.62$\\
\bottomrule
\end{tabular}}
\end{table}

The $q90$ column shows the same ordering as the median: the sequential EKF
and endpoint-only registration are not just worse on average but have
heavier upper tails ($0.172$\,m and $0.169$\,m q90 vs.\ $0.117$\,m for the
proposed estimator), and the fixed-lag smoother's q90 is statistically
indistinguishable from the proposed estimator's at roughly $50\times$ the
runtime cost, reinforcing that the closed-form registration is
information-sufficient rather than merely competitive on the median.

\subsection{Disturbance Recovery: Denominators and Wilson Intervals}

Table~\ref{tab:disturbancefull} reports the full recovery statistics
underlying Fig.~\ref{fig:disturbance} and the main text's disturbance
paragraph, across all three tested mission phases (control: no step;
station: at rest; transit: mid-transit) and both excitation variants, with
exact trial counts and Wilson 95\% intervals. ``Adoption'' is the
calibration event (persistent certified inconsistency triggers a new
registration); ``task recovery'' is the separate, later event of task
error returning below tolerance. The two are not interchangeable, and
station-phase steps show the largest gap between them because the task
loop has no incentive to move (Remark~\ref{prop:recovery}, Scope) until
the stale calibration itself causes task error.

\begin{table}[h]
\caption{Disturbance recovery by mission phase and step size, supervised
excitation variant ($n$ per phase$\times$step cell; Wilson 95\% intervals).
Adoption/recovery times in seconds; ``--'' where no adoption is required.}
\label{tab:disturbancefull}
\centering
\footnotesize
\setlength{\tabcolsep}{2pt}
\resizebox{\columnwidth}{!}{%
\begin{tabular}{@{}llcccccc@{}}
\toprule
Phase & $\Delta$ & $n$ & Adoption rate [Wilson 95\%] & Task recovery [Wilson 95\%] & Detect delay [s] & Recover [s]\\
\midrule
Control & $0^\circ$ & 200 & $0.000$ $[0,0.019]$ & -- & -- & --\\
Station & $20^\circ$ & 150 & $0.120$ $[0.077,0.182]$ & $0.110$ $[0.069,0.172]$ & $20.85$ & $21.25$\\
Station & $40^\circ$ & 150 & $0.700$ $[0.622,0.768]$ & $0.667$ $[0.588,0.737]$ & $20.90$ & $20.90$\\
Station & $80^\circ$ & 150 & $1.000$ $[0.975,1]$ & $0.980$ $[0.943,0.993]$ & $21.05$ & $21.05$\\
Transit & $20^\circ$ & 150 & $0.673$ $[0.595,0.743]$ & $0.745$ $[0.668,0.809]$ & $4.20$ & $4.15$\\
Transit & $40^\circ$ & 150 & $0.773$ $[0.700,0.833]$ & $0.780$ $[0.707,0.839]$ & $6.30$ & $5.25$\\
Transit & $80^\circ$ & 150 & $1.000$ $[0.975,1]$ & $1.000$ $[0.975,1]$ & $7.15$ & $7.15$\\
\bottomrule
\end{tabular}}
\end{table}

Every cell retains $100\%$ task success (Table~\ref{tab:odom}'s success
column is unaffected by whether adoption occurs); what varies is only
whether and when the certificate re-calibrates. The control row's zero
adoptions in $200$ undisturbed trials gives the Wilson upper bound of
$1.9\%$ cited in the main text as the empirical false-adoption rate.

\subsection{Long-Horizon Distributions}

Table~\ref{tab:longhorizonfull} adds the $90$th-percentile station error
alongside the median dead-reckoning drift already shown in
Fig.~\ref{fig:longhorizon}, over the $600$\,s campaign at each tested body
bias ($100$ trials/cell).

\begin{table}[h]
\caption{600\,s long-horizon station error and dead-reckoning drift by
body-frame velocity bias (100 trials/cell).}
\label{tab:longhorizonfull}
\centering
\footnotesize
\setlength{\tabcolsep}{4pt}
\begin{tabular}{@{}cccc@{}}
\toprule
Bias [cm/s] & Med.\ RMSE [m] & q90 RMSE [m] & Med.\ dead-reck.\ [m]\\
\midrule
$0$   & $0.0637$ & $0.108$ & $1.39$\\
$1$   & $0.0618$ & $0.105$ & $4.98$\\
$2$   & $0.0607$ & $0.105$ & $10.97$\\
$5$   & $0.0577$ & $0.108$ & $27.20$\\
\bottomrule
\end{tabular}
\end{table}

Station error is flat across nearly two orders of magnitude of induced
drift: the q90 station error at $5$\,cm/s bias ($0.108$\,m) is
statistically indistinguishable from the bias-free case ($0.108$\,m),
while median dead-reckoning error grows from $1.4$\,m to $27.2$\,m over the
same range: the closed loop's error is set by the certificate's finite
calibration window, not by how far dead reckoning has drifted, exactly the
horizon-independence of Theorem~\ref{thm:seeking}.

\subsection{Gazebo Physics-Based Validation, Expanded}

Table~\ref{tab:gazebofull} expands Table~\ref{tab:gazebo} with
interquartile range, maximum error, and odometry-frame position/yaw RMSE
for all three scenarios ($30$ independent launches each, commit
\texttt{0af4309}). The combined-stress scenario applies $30\%$ packet
dropout, $0.2$\,s fixed delay, $0.1$\,s jitter, and $5\%$ gross outliers
simultaneously, each factor individually milder than the corresponding
offline one-factor sweep (Table~\ref{tab:comm}), but combined rather than
isolated, which is why station RMSE ($0.164$\,m) exceeds any single
offline stress cell despite the milder per-factor levels.

\begin{table}[h]
\caption{Gazebo validation, expanded (30 launches/scenario). Odom.\ RMSE is
against the simulator's own drive-train integration, not ground truth.}
\label{tab:gazebofull}
\centering
\footnotesize
\setlength{\tabcolsep}{2pt}
\resizebox{\columnwidth}{!}{%
\begin{tabular}{@{}lcccc@{}}
\toprule
Scenario & Station RMSE [m] & Max [m] & Odom.\ pos.\ [m] & Odom.\ yaw [rad]\\
\midrule
Nominal & $0.077$ $[0.051,0.094]$ & $0.139$ & $0.083$ $[0.067,0.097]$ & $0.015$ $[0.011,0.027]$\\
Comm.\ stress & $0.164$ $[0.149,0.198]$ & $0.238$ & $0.105$ $[0.064,0.152]$ & $0.100$ $[0.058,0.158]$\\
Disturbance & $0.043$ $[0.031,0.070]$ & $0.109$ & $0.064$ $[0.043,0.085]$ & $0.042$ $[0.034,0.050]$\\
\bottomrule
\end{tabular}}
\end{table}

The disturbance scenario's station RMSE is computed over the final
$30$\,s (post-recovery, matching Table~\ref{tab:gazebo}'s caption) and is
lower than nominal because the vehicle spends more of that window already
converged after the mid-transit step; its odometry-frame yaw RMSE
($0.042$\,rad) reflects ordinary drive-train integration error, not the
$60^\circ$ injected relay-frame step itself, which is a control-frame event
the odometry never observes.

\FloatBarrier

\bibliographystyle{IEEEtran}
\bibliography{references}

\end{document}